\documentclass[a4,smallextended]{svjour3}       \smartqed  \usepackage{graphicx}
\usepackage{amssymb,dsfont,amsmath,mathtools,natbib}
\usepackage{xcolor}
\newcommand{\dotprod}[2]{\ensuremath{\langle #1 , #2\,\rangle}} 
\usepackage{url}
\usepackage{algorithm}
\usepackage{algorithmic}
\newcommand{\norm}[1]{|\!| #1 |\!|}
\def\dbR{\mathbb{R}}
\def\x{\mathbf{x}}
\def\X{\mathbf{X}}
\def\u{\mathbf{u}}
\def\v{\mathbf{v}}

\def\Z{\mathbf{Z}}
\def\E{\mathbf{E}}
\def\K{\mathbf{K}}
\def\z{\mathbf{z}}
\def\t{\mathbf{t}}
\def\M{\mathbf{M}}
\def\T{\mathbf{T}}
\def\P{\mathbf{P}}
\def\p{\mathbf{p}}
\def\C{\mathbf{C}}

\def\I{\mathbf{I}}
\def\L{\mathbf{L}}
\def\zero{\mathbf{0}}
\def\balpha{\boldsymbol{\alpha}}
\def\bbeta{\boldsymbol{\beta}}

\newcommand{\one}{{\mathbf{1}}} 
\DeclareMathOperator{\tr}{tr}

\DeclareMathOperator{\argmin}{argmin}

\DeclareMathOperator{\diag}{diag}

\newcommand{\rev}[1]{{\color{black}#1}}

 \journalname{Machine Learning Journal}

\begin{document}

\title{Wasserstein Discriminant Analysis
\thanks{
This work was supported in part by grants from the ANR OATMIL ANR-17-CE23-0012, Normandie
Region, Feder, CNRS PEPS DESSTOPT, Chaire d'excellence de l'IDEX Paris Saclay.\\
The authors also want to thank Alexandre Saint-Dizier, Charles Bouveyron and Julie Delon for fruitful discussion and for pointing out the class weights in Fisher Discriminant.
}}
\subtitle{}

\author{R\'emi Flamary         \and Marco Cuturi \and Nicolas Courty
  \and Alain Rakotomamonjy
         }

\institute{R. Flamary \at
              Lagrange, Observatoire de la C\^ote d'Azur\\
              Universit\'e  C\^ote d'Azur \\
              \email{remi.flamary@unice.fr}                      \and
           M. Cuturi \at
           CREST, ENSAE \\
		   Campus Paris-Saclay
		   5, avenue Henry Le Chatelier
		   91120 Palaiseau, France
           \and
           N. Courty \at
           Laboratoire IRISA \\
           Campus de Tohannic \\
           56000 Vannes, France	\\
           \and
           A. Rakotomamonjy \at
           LITIS EA4108,  \\
           Universit\'e Rouen Normandie
}

\date{Received: date / Accepted: date}

\maketitle

\begin{abstract}
Wasserstein Discriminant Analysis (WDA) is a new supervised linear
dimensionality reduction algorithm. Following the
blueprint of classical Fisher Discriminant Analysis, WDA selects
the projection matrix that maximizes the ratio of the
dispersion of projected points pertaining to different classes and
the dispersion of projected points belonging to a same class. To
quantify dispersion, WDA uses regularized Wasserstein distances.
 Thanks to the underlying principles of
optimal transport, WDA is able to capture both global (at distribution
scale) and local (at samples' scale) interactions between
classes. In addition, we show that WDA leverages a mechanism
that induces neighborhood preservation.
 Regularized Wasserstein distances can be computed using the
Sinkhorn matrix scaling algorithm;  the optimization problem of
WDA can be tackled using automatic differentiation of Sinkhorn's fixed-point
iterations. Numerical experiments show promising results both in terms
of prediction and visualization on toy examples and real datasets
such as MNIST and on deep features obtained from a subset of the
Caltech dataset. 
\end{abstract}

\section{Introduction}
\label{sec:introduction}

Feature learning is a crucial component in many applications of machine learning.
New feature extraction methods or data representations are often responsible for breakthroughs in performance, as illustrated by the
kernel trick in support vector machines \citep{scholkopf2002learning} and their feature
learning counterpart in multiple kernel
learning \citep{bach2004multiple}, and more recently by deep
architectures \citep{bengio2009learning}. 

Among all the feature extraction approaches, one major family of
dimensionality reduction methods \citep{van2009dimensionality,burges2010dimension}
consists in estimating a linear subspace of the data. \rev{Although very simple, linear subspaces have many advantages. They are easy to
interpret, and can be inverted, at least in a lest-squares way. This latter property has been used for instance in PCA
denoising \citep{zhang2010two}. }
Linear projection is also a key component in random
projection methods \citep{fern2003random} or compressed sensing and is
often used as a first pre-processing step, such as the linear part in
a neural network layer. \rev{Finally, linear projections only imply matrix products and stream therefore particularly well on any type of hardware (CPU,
GPU, DSP).}

 Linear dimensionality reduction techniques come in all flavors. Some of them, such as PCA, are
inherently unsupervised; some can consider labeled
data and fall in the supervised category. We consider in this paper
\emph{linear} and \emph{supervised} techniques. Within that category,
two families of methods stand out: Given a dataset of pairs of vectors
and labels $\{(\x_i,y_i)\}_i$, with $\x_i \in \mathbb{R}^d$, the goal
of \emph{Fisher Discriminant Analysis} (FDA) and variants is to learn
a linear map $\P:\mathbb{R}^d \rightarrow \mathbb{R}^p$, $p\ll d$,
such that the embeddings of these points $\P \x_i$ can be easily
discriminated using linear classifiers. \emph{Mahalanobis metric
  learning} (MML) follows the same approach, except that the quality
of the embedding $\P$ is judged by the ability of a $k$-nearest
neighbor algorithm (not a linear classifier) to obtain good
classification
accuracy. 
\paragraph{FDA and MML, in both Global and Local Flavors.} FDA  attempts to maximize w.r.t. $\P$ the sum of \emph{all} distances
$\norm{\P\x_i-\P\x_{j'}}$ between pairs of samples from different
classes $c,c'$ while minimizing the sum of \emph{all} distances
$\norm{\P\x_i-\P\x_j}$ between pairs  of samples within the same class
$c$ \citep[\S4.3]{friedman2001elements}.
Because of this, it is well documented that the performance of FDA degrades when class distributions are multimodal. Several variants have been proposed to tackle this problem~\citep[\S12.4]{friedman2001elements}. For instance, a localized version of FDA was proposed by \citet{sugiyama2007dimensionality}, which boils down to discarding the computation for all pairs of points that are not neighbors.
\begin{figure*}[t]
  \centering
 \includegraphics[height=2.6cm]{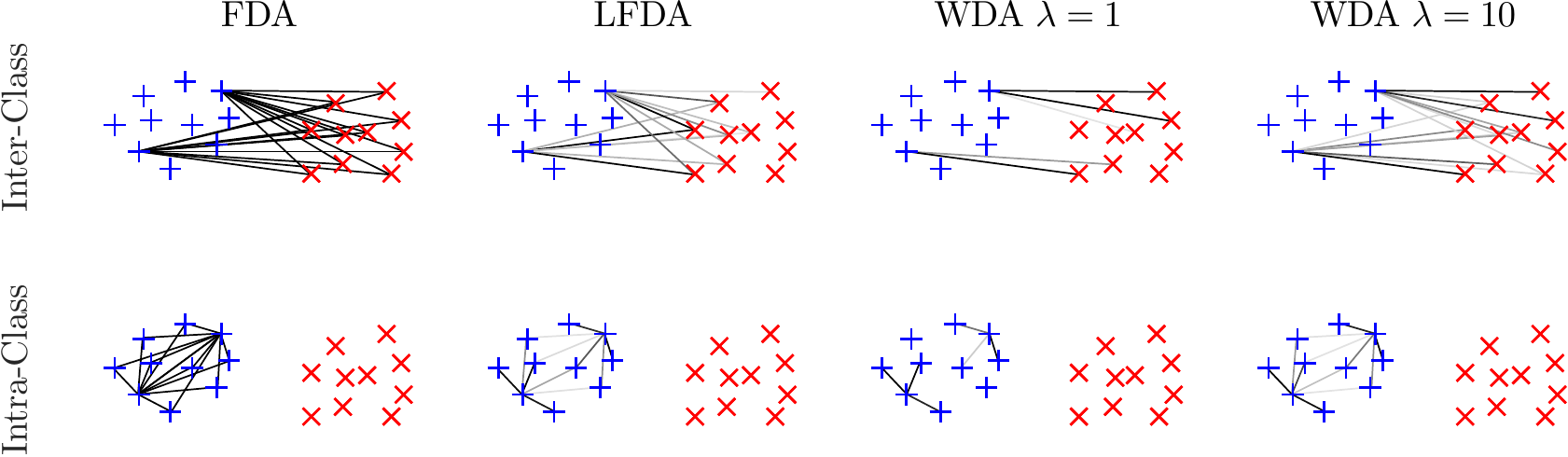}\hspace{5mm}
\includegraphics[height=2.6cm]{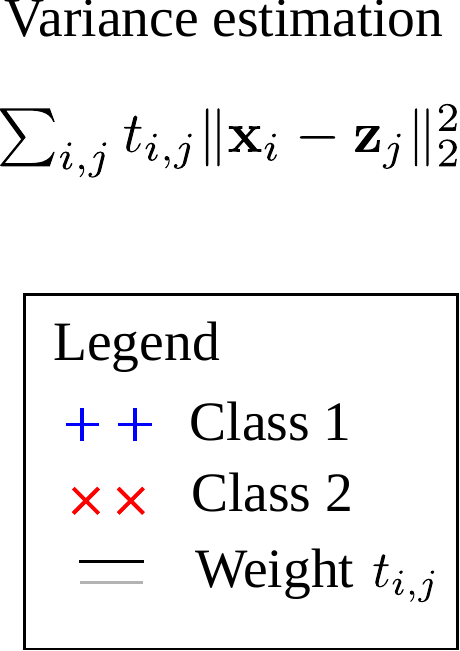}
  \caption{Weights used for inter/intra class
    variances for FDA, Local FDA and WDA for different regularizations $\lambda$. Only weights for two
    samples from class 1 are shown. The color of the link darkens as the weight grows. FDA computes a global
    variance with uniform weight on all pairwise distances, whereas LFDA focuses only on samples that lie close to each other. WDA relies on an optimal transport matrix $\T$ that matches all points in one class to all other points in another class (most links are not visible because they are colored in white as related weights are too small). WDA has both a global (due to matching constraints) and local (due to transportation cost minimization) outlook on the problem, with a tradeoff controlled by the regularization strength $\lambda$.}
  \label{fig:regwass}
\end{figure*}
On the other hand, the first techniques that were proposed to learn metrics~\citep{xing2003distance} used a \emph{global} criterion, namely a sum on all pairs of points. Later on, variations that focused instead exclusively on \emph{local} interactions, such as LMNN~\citep{weinberger2009distance}, were shown to be far more efficient in practice. Supervised dimensionality approaches stemming from FDA or MML consider thus \emph{either} global \emph{or} local interactions between points, namely, either all differences $\norm{\P\x_i-\P\x_j}$ have an equal footing in the criterion they optimize, or, on the contrary, $\norm{\P\x_i-\P\x_j}$ is only considered for points such that $\x_i$ is close to $\x_j$. 

\paragraph{WDA: Global and Local.} We introduce in this work a novel
approach that incorporates both a global and local perspective. WDA
can achieve this blend through the mathematics of optimal transport
(\rev{ see for instance the recent book of~\cite{peyre18}
for an introduction and exposition of some of the computational we will use in this paper). Optimal transport provides a powerful toolbox to compute distances
between two empirical probability distributions. Optimal transport does so by considering all probabilistic couplings
between these two measures, to select one, denoted $T$, that is optimal for a given criterion. }This coupling now describes interactions at both a global and local scale, as reflected by the transportation weight $T_{ij}$ that quantifies how important the distance
$\norm{\P\x_i-\P\x_{j}}$ should be to obtain a good projection matrix $\P$.  Indeed,
such weights are decided by \rev{\emph{(i)} making sure that \emph{all}
points in one class are matched to \emph{all} points in the other class
(global constraint\rev{, derived through marginal constraints over the coupling}); \emph{(ii)} making sure that points in one class are  matched only to few \rev{similar} points of the other class \rev{(local constraint, thanks to the optimality of the
coupling, that is a function of local costs}). }
Our method has the added flexibility that it can interpolate, through a regularization parameter, between an
exclusively global viewpoint (identical, in that case, to FDA), to a
more local viewpoint with a global matching constraint (different, in
that sense, to that of purely local tools such as LMNN or
Local-FDA). In mathematical terms, we adopt the ratio formulation of
FDA to maximize the ratio of the regularized Wasserstein distances
between inter class populations and between the intra-class population
with itself, when these points are considered \emph{in their projected
  space}:
 \begin{equation}
  \max_{\P\in\Delta}\quad\frac{\sum_{c,c'> c}W_\lambda(\P\X^{c},\P\X^{c'})}{\sum_{c}W_\lambda(\P\X^{c},\P\X^{c})}\label{eq:wdaprob}
\end{equation}
where $\Delta=\{ \P=[\mathbf{p_1},\dots,\mathbf{p}_p] \quad|\quad \p_i \in \dbR^d, 
\|\mathbf{p}_i\|_2=1 \quad\text{ and
}\p_i^\top\p_j=0\text{ for } i\neq j \}$ is the Stiefel manifold \citep{absil2009optimization}, the set of orthogonal
$d\times p$ matrices;
  $\P\X^{c}$ is the matrix of projected samples from class $c$.
$W_\lambda$ is the  regularized Wasserstein
distance proposed by \citet{CuturiSinkhorn}, which can be expressed as
$W_\lambda(\X,\Z)=\sum_{i,j}T^\star_{i,j}\|\x_i-\z_j\|^2_2$, $T^\star_{i,j}$
being the coordinates of the
entropic-regularized Optimal Transport (OT) matrix $\T^\star$ (see \S\ref{sec:regul-wass-dist}). 
\rev{
These entropic-regularized Wasserstein distances 
measure the dissimilarity between empirical distributions by considering  pairwise distances between samples. The strength of the regularization $\lambda$ controls the local information involved in the distance computation. 
} \rev{Further analyses and intuitions 
on the role on the within-class distances in the optimization problem
are given in the Discussion section.}

When $\lambda$ is small, we will show that WDA boils down to FDA.
When $\lambda$ is large, WDA tries to split apart distributions of classes by maximizing their optimal transport distance. 
In that process,  for a given example $\x_i$  in one class, only few
components $T_{i,j}$ will be activated so that  $\x_i$ will be paired
with few examples. Figure \ref{fig:regwass} illustrates how pairing
weights $T_{i,j}$ are defined when comparing Wasserstein discriminant
analysis (WDA, with different regularization strengths) with FDA
(purely global), and Local-FDA (purely local)
\citep{sugiyama2007dimensionality}. Another strong feature brought by
regularized Wasserstein distances 
is that relations between samples (as given by the optimal transport
matrix $\T$) are estimated in the projected space. This is an
important difference compared to all previous local approaches which
estimate local relations 
in the original space and make the hypothesis that these
relations are unchanged after projection.

\paragraph{Paper outline.} Section \ref{sec:theor-backgr} provides 
background on regularized Wasserstein distances. The WDA  criterion
and its practical optimization is presented in
Section \ref{sec:wass-discr-analys}.  Section \ref{sec:discussions} by
discusses properties of WDA and related works. Numerical experiments
are provided in Section~\ref{sec:numer-exper}. Section \ref{sec:conclusion}
concludes the paper and introduces perspectives.

\section{Background on Wasserstein distances}
\label{sec:theor-backgr}

Wasserstein distances, also known as earth mover
distances, define a geometry over the space of probability measures using principles from optimal transport theory~\citep{villani08}. Recent computational advances \citep{CuturiSinkhorn,benamou2015iterative} have made them scalable to dimensions relevant to machine learning applications.

\paragraph{Notations and Definitions.}
Let $\mu=\frac{1}{n}\sum_{i} \delta_{\x_i}$, $\nu=\frac{1}{m}\sum_{i}
\delta_{\z_i}$ be two empirical measures with locations in
$\mathbb{R}^d$ stored in matrices
$\X=[\x_1,\cdots,\x_n]$ and $\Z=[\z_1,\cdots,\z_m]$. The pairwise squared Euclidean distance matrix between samples in $\mu$ and $\nu$ is defined as
$\M_{\X,\Z}:= [\norm{\x_i-\z_j}_2^2]_{ij} \in\mathbb{R}^{n\times m}.$
Let $U_{nm}$ be the polytope of $n\times m$ nonnegative matrices such that their row and column marginals are equal to $\mathbf{1}_n/n$ and $\mathbf{1}_m/m$ respectively. 
Writing $\one_n$ for the $n$-dimensional vector of ones, we have:
\begin{equation*}\label{eq:polytope} 
	U_{nm}:= \{ \T\in\mathbb{R}_+^{n\times m}\;:\; \T \mathbf{1}_m = \mathbf{1}_n/n,\, \T^T \mathbf{1}_n = \mathbf{1}_m/m \}.
\end{equation*}

\paragraph{Regularized Wassersein distance.}
\label{sec:regul-wass-dist}
Let $\dotprod{\mathbf{A}}{\mathbf{B}}:= \tr(\mathbf{A}^T \mathbf{B})$ be the Frobenius dot-product of matrices. 
For $\lambda\geq 0$, the regularized Wasserstein distance we adopt in this paper between $\mu$ and $\nu$ is (and with a slight abuse of notation):
\begin{equation}\label{eq:sink}W_\lambda(\mu,\nu) :=W_\lambda(\X,\Z) :=\dotprod{\T_\lambda}{\M_{\X,\Z}},\end{equation}
	where $\T_\lambda$ is the solution of an entropy-smoothed optimal transport problem,
\begin{equation}\label{eq:primal_sol}\T_\lambda :=\argmin_{\T\in
    U_{nm}}\,\lambda\dotprod{\T}{\M_{\X,\Z}}-\Omega(\T),\end{equation}
where $\Omega(\T)$ is the entropy of $\T$ seen as a discrete joint probability distribution, namely $\Omega(\T):=-\sum_{ij}t_{ij} \log(t_{ij}).$
Note that problem (\ref{eq:primal_sol}) can be solved very efficiently
using Sinkhorn's
 fixed-point iterations \citep{CuturiSinkhorn}. The
solution of the optimization problem can be expressed as:
\begin{equation}\label{eq:solregwda}
\rev{
  \T=\diag(\u) \K\diag(\v)
=\u\one_m^T \circ \K  \circ\one_n {\v}^T, }
\end{equation}
where $\circ$ stands for elementwise multiplication 
\rev{and $\K$ is
the matrix whose elements are $K_{i,j} = e^{-\lambda M_{i,j}}$.}
 The Sinkhorn iterations
consist in updating left/right scaling vectors $\u^k$ and $\v^k$ of
the matrix $\K=e^{-\lambda\M}$. These updates take the following form
for iteration $k$:
\begin{align}
  \v^k=\frac{\one_m/m}{{\K}^T\u^{k-1}}, \quad \u^k=\frac{\one_n/n}{\K\v^k} \label{eq:fixedpoint}
\end{align}
 with an initialization which will be fixed to $\u^0=\one_n$. Because it only involves matrix products, the Sinkhorn algorithm can be streamed efficiently on parallel architectures such as GPGPUs.

\section{Wasserstein Discriminant Analysis}
\label{sec:wass-discr-analys}

In this section we discuss optimization problem~\eqref{eq:wdaprob} and propose an efficient approach to compute the gradient of its objective.

\subsection{Optimization problem.} To simplify notations, let us define a separate
empirical measure for each of the $C$ classes:  samples of 
class $c$ are stored in
matrices $\X^c$; the number of samples from class $c$ is $n_c$.
Using the definition~\eqref{eq:sink} of regularized Wasserstein
distance, we can write the Wasserstein Discriminant
Analysis optimization problem as
\begin{align}
 \label{eq:wdabilevel} 
\max_{\P\in\Delta} \quad &\left\{J(\P,\T(\P))=\frac{\sum_{c,c'> c}
   \dotprod{\P^T\P}{\C^{c,c'}}
}{\sum_{c}
   \dotprod{\P^T\P}{\C^{c,c}}}\right\}&\\\vspace{-2mm}\nonumber
 \text{s.t. }  \C^{c,c'}&=\sum_{i,j}T^{c,c'}_{i,j}(
  \x^c_i-\x^{c'}_j)( \x^c_i-\x^{c'}_j)^T,\quad  \forall c,c'\\\nonumber
\text{and } \T^{c,c'}&=\small\argmin_{\T\in
                        U_{n_cn_{c'}}}\,\lambda\dotprod{\T}{\M_{\P\X^{c},\P\X^{c'}}}-\Omega(\T),\nonumber                      \end{align}
which can be reformulated as the following bilevel problem
\begin{align}
&  \max_{\P\in\Delta} \quad\quad J(\P,\T(\P))\label{eq:outer}\\
&\text{s.t. } \T(\P)=\argmin_{\T\in U_{n_cn_{c'}}}\quad E(\T,\P) \label{eq:inner}
\end{align}
where $\T=\{\T^{c,c'}\}_{c,c'}$ contains all the transport matrices
between classes and the inner problem
function $E$  is  defined as
\begin{equation}
 E(\T, \P)=\sum_{c,c>=c'}
\lambda\dotprod{\T^{c,c'}}{\M_{\P\X^{c},\P\X^{c'}}}-\Omega(\T^{c,c'}). \label{eq:jloss}
\end{equation}
The objective function $J$ can be expressed as 
$$J(\P,\T(\P))= \frac{
   \dotprod{\P^T\P}{\C_{b}}
}{
   \dotprod{\P^T\P}{\C_{w}}}$$
where $\C_b=\sum_{c,c'> c}\C_{c,c'}$ and $\C_w=\sum_{c} \C_{c,c}$ are
the between and within cross-covariance matrices that depend on
$\T(\P)$. Optimization problem \eqref{eq:outer}-\eqref{eq:inner} is a bilevel optimization problem, which can be solved using gradient descent
\citep{colson2007overview}. Indeed, $J$ is differentiable
with respect to $\P$.  This comes from the fact that 
optimization problems in Equation \eqref{eq:inner} are all strictly convex,
making solutions of the problems unique, hence 
$\T(\P)$ is smooth and differentiable \citep{bonnans1998optimization}.

Thus, one can
compute the gradient of $J$ directly \emph{w.r.t.} $\P$ using the chain rule as follows
\begin{equation}\small
\nabla_\P J(\P,\T(\P))=\frac{\partial J(\P,\T)}{\partial \P }
+\sum_{c,c'\geq c} \frac{\partial J(\P,\T)}{\partial \T^{c,c'}}
\frac{\partial \T^{c,c'}}{\partial \P}
\label{eq:fullgrad}
\end{equation}
The first term in gradient \eqref{eq:fullgrad} suppose that $\T$ is
constant and can be computed (Eq. 94-95 \citep{petersen2008matrix}) as
\begin{equation}
\frac{\partial J(\P,\T)}{\partial \P
}=\P\left(\frac{2}{\sigma_w^2}\C_{b}-\frac{2\sigma_b^2}{\sigma_w^4}\C_{w}\right)\label{eq:djdp}
\end{equation}
with $\sigma_w^2=\dotprod{\P^T\P}{\C_{w}}$ and
$\sigma_b^2=\dotprod{\P^T\P}{\C_{b}}$.
In order to compute the second term in \eqref{eq:fullgrad}, we will separate the cases when
$c=c'$ and $c\neq c'$ as it corresponds to their position in the
fraction of Equation (\ref{eq:wdabilevel}).  Their partial
derivative is obtained directly from the scalar product and is a
weighted vectorization of the transport cost matrix 
\begin{equation}\label{eq:djdt}
 \frac{\partial J(\P,\T)}{\partial \T^{c,c'\neq
  c}}=\text{vec} \Big(\frac{1}{{\sigma_w^2}}\M_{\P\X^{c},\P\X^{c'}}\Big) \quad
\text{and} \quad \frac{\partial J(\P,\T)}{\partial \T^{c,c}}=-\text{vec}\Big({\frac{\sigma_b^2}{\sigma_w^4}}\M_{\P\X^{c},\P\X^{c}}\Big).
\end{equation}

We will see in the
remaining that the main difficulty stands in
computing the Jacobian $ {\partial \T^{c,c'}}/{\partial \P}$
since the optimal transport matrix is not available as a closed form.
We solve this problem using instead an automatic differentiation approach
wrapped around the Sinkhorn fixed point iteration algorithm.

\subsection{Automatic Differentiation.}
\label{sec:grad-comp-using}
A possible way to compute the Jacobian $\partial \T^{c,c'}/\partial
\P$ is to use the implicit function theorem as in
hyperparameter estimation in ML
\citep{bengio2000gradient,chapelle2002choosing}. 
We detail that approach in the appendix but it requires inverting a
very large matrix, and does not scale in practice. It also assumes that the exact optimal transport $\T_\lambda$ is obtained at each iteration, which is clearly an approximation since we only have the computational budget for a finite, and usually small, number of Sinkhorn
iterations. 

Following the gist of \citet{bonneel2016wasserstein}, which
do not differentiate Sinkhorn iterations but a more complex fixed
point iteration designed to compute Wasserstein barycenters, we
propose in this section to differentiate the transportation matrices
obtained after running exactly $L$ Sinkhorn iterations, with a
predefined $L$. Writing $\T^k(\P)$, for the solution obtained after $k$
iterations as a function of $\P$ for a given $c,c'$ pair,
$$
  \T^k(\P) = \diag(\u^k)e^{-\lambda\M}\diag(\v^k)
$$
where $\M$ is the distance matrix induced by $\P$. $\T^L(\P)$ can then be directly differentiated:
\begin{align}\label{eq:gradT}
\frac{\partial \T^k}{\partial \P} & = \frac{\partial [\u^k\one_m^T]}{\partial \P} \circ e^{-\lambda\M}\circ\one_n {\v^k}^T \\\nonumber
 +&\u^k\one_m^T \circ \frac{\partial e^{-\lambda\M}} {\partial \P}\circ\one_n {\v^k}^T 
+ \u^k\one_m^T \circ e^{-\lambda\M}\circ\frac{\partial{ [\one_n\v^k}^T]}{\partial \P}  
\end{align}
Note that the recursion occurs as $\u^k$ depends on $\v^k$ whose is also
related to $\u^{k-1}$. The Jacobians that we need can then be obtained from Equation~\eqref{eq:fixedpoint}. For instance, the gradient of
one component of $\u^k\one_m^T$ at the $j$-th line  is
\begin{equation}\label{eq:gradu}
\frac{\partial \u^k_j}{\partial \P}=
-\frac{1/n}{[\K \v^k]_j^2}\Big(\sum_i \frac{\partial \K_{j,i}}{ \partial \P}\v_i^k
+  \sum_i \K_{j,i} \frac{\partial \v_i^k}{ \partial \P} \Big),
\end{equation}
while for $\v^k$, we have
\begin{equation}\label{eq:gradv}
\frac{\partial \v^k_j}{\partial \P}=
-\frac{1/m}{[\K^T \u^{k-1}]_j^2}\Big(\sum_i \frac{\partial \K_{i,j}}{ \partial \P}\u_{i}^{k-1}
+  \sum_i \K_{i,j} \frac{\partial \u_i^{k-1}}{ \partial \P} \Big),
\end{equation}
and finally
$$
\frac{\partial \K_{i,j}}{\partial \P}= - 2K_{i,j} \P(\x_i - \x_j^\prime)
(\x_i - \x_j^\prime)^T.
$$
The Jacobian $\frac{\partial \T^k}{\partial \P}$ can be thus obtained by keeping track of 
all the Jacobians at each iteration and then by successively applying
those equations.  This approach is far cheaper than the implicit function theorem
approach. Indeed, in this case, the computation of
$\frac{\partial \T}{\partial \P}$ is dominated by the complexity of computing
$\frac{\partial \K}{\partial \P}$ whose costs for one iteration is
$O(pn^2d^2)$ for
$n=m$. The complexity is then linear in $L$ and quadratic in $n$.

\rev{
\subsection{Algorithm}

In the above subsections, we have reformulated the WDA optimization problem
so as to make it tractable. We have derived closed-form expressions of
some elements of the gradient as well as an automatic differentiation
strategy for computing gradients of the transport plans $\T^{c,c^\prime}$ with respects to $\P$.    

Now that all these partial derivatives are computed, we can compute 
the gradient  $\mathbf{G}^k = \nabla_\P J(\P^k, \T(\P^k))$ 
at iteration $k$ and apply  classical manifold optimization tools such as
projected gradient of \citet{schmidtminconf} or a trust region algorithm as
implemented in Manopt/Pymanopt \citep{boumal2014manopt,koep2016pymanopt}. The latter toolbox includes tools to optimize over the Stiefel manifold, notably automatic conversions from
Euclidean to Riemannian gradients.
Algorithm \ref{alg:projectedgrad} provide the steps of a projected gradient 
approach for solving WDA. We noted in there that at each iteration, we need
to compute all the transport plans $\T^{c,c^\prime}$, which are needed
for computing $\C_b$ and $\C_w$. Automatic differentiation in the last
steps of Algorithm \ref{alg:skad} takes advantage of these transport
plan computations for calculating and storing partial derivatives
needed for Equation \ref{eq:gradT}.  

From a computational complexity point of view, for each projected gradient iteration, we have the following complexity, considering that all classes are
composed of $n$ samples.  For one iteration of the Sinkhorn
-Knopp algorithm given in Algorithm \ref{alg:skad}, $\u^k$ and $\v^k$ are of complexity  $\mathcal{O}(n^2)$ while   $\{\frac{\partial \v_{j}^k}{ \partial \P}\}_j$
and $\{\frac{\partial \u_{j}^k}{ \partial \P}\}_n$ are both
 $\mathcal{O}(n^2dp)$. In this algorithm, complexity is dominated by
the one of $\frac{\partial \K_{i,j}}{ \partial \P}$ which costs is $\mathcal{O}(n^2d^2p)$, although it is computed only once. 
In Algorithm \ref{alg:projectedgrad}, the costs of $\C_b$ and $\C_w$ are 
$\mathcal{O}(n^2d^2)$ and Equation~(\ref{eq:djdp})~and~(\ref{eq:djdt}) 
yields respectively a complexity of $\mathcal{O}(pd^2)$ and $\mathcal{O}(n^2d^2p)$. Note that the cost of computing $\frac{\partial \T^k}{\partial \P}$ is
dominated by those of   $\{\frac{\partial \v_{j}^k}{ \partial \P}\}_j$
and $\{\frac{\partial \u_{j}^k}{ \partial \P}\}_n$. Finally, the cost
 computing the sum in $\mathbf{G}^k = \nabla_\P J(\P^k, \T(\P^k)$ achieves 
a global complexity of $\mathcal{O}(C^2n^2dp)$. In conclusion, our algorithm
is quadratic in both the number of samples in the classes and in the original
dimension of the problem. 

\begin{algorithm}[t]
\caption{Projected gradient algorithm for WDA}\label{alg:projectedgrad}
  \begin{algorithmic}[1]
\REQUIRE $\Pi_\Delta$ : projection on the Stiefel manifold
  \STATE Initialize $k = 0$, $\P^0$
  \REPEAT
  \STATE compute all the $\T^{c,c^\prime}$ as given in Equation~(\ref{eq:inner}) by means of Algorithm \ref{alg:skad}
  \STATE compute  $\C_b$ and $\C_w$
  \STATE compute Equation (\ref{eq:djdp}) for $\P^k$
  \STATE compute Equation (\ref{eq:djdt}) for $\P^k$
  \STATE compute $\frac{\partial \T^k}{\partial \P}$ using automatic differentiation based on Equations (\ref{eq:gradT}), (\ref{eq:gradu}) and (\ref{eq:gradv}) 
  \STATE compute gradient $\mathbf{G}^k = \nabla_\P J(\P^k, \T(\P^k)$ using all above elements
  \STATE compute descent direction $ \mathbf{D}^k = \Pi_\Delta (\P^k - \mathbf{G}) - \P^k$
\STATE  linesearch on the step-size $\alpha_k$  
\STATE $\P^{k+1} \leftarrow \Pi_\Delta (\P^k + \alpha_k  \mathbf{D}^k$) 
\STATE $k \leftarrow k +1$  
\UNTIL{convergence}
\end{algorithmic}
\end{algorithm}

\begin{algorithm}[t]
\caption{Sinkhorn-Knopp algorithm with automatic differentiation \label{alg:skad}}
  \begin{algorithmic}[1]
\REQUIRE $\K = e^{-\lambda \M_{\P\X^c,\P\X^{c^\prime}}}$, $L$ the number of iterations
\STATE Initialize $k = 0$, $\u^0 =  \mathbf{1}_n$, $\frac{\partial \u_{j}^0}{ \partial \P} = \mathbf{0}$ for all $j$
\STATE compute $\frac{\partial \K_{i,j}}{ \partial \P}$ \COMMENT{store these gradients for computing $\frac{\partial \T^k}{\partial \P}$}
\FOR { $k = 1$ to  $L$}
\STATE compute $\mathbf{v^k}$ and $\mathbf{u^k}$ as given in Equation (\ref{eq:fixedpoint})
\STATE compute $\frac{\partial \v_{j}^k}{ \partial \P}$ for all $j$ \COMMENT{store these gradients for computing $\frac{\partial \T^k}{\partial \P}$}
\STATE compute $\frac{\partial \u_{j}^k}{ \partial \P}$ for all $j$ \COMMENT{store these gradients for computing $\frac{\partial \T^k}{\partial \P}$}

\ENDFOR
\OUTPUT $\u^k$, $\v$ and all the gradients
\end{algorithmic}
\end{algorithm}

}

\section{Discussions}
\label{sec:discussions}

\subsection{Wasserstein Discriminant analysis : local and global.}

As we have stated, WDA allows construction of both
a local and global interaction of the empirical distributions
to compare. Globality naturally results  from the Wasserstein
distance, which is a metric on probability measures, and as 
such it measures discrepancy between distributions at whole
level. Note however that this property would have been shared
by any other metric on probability measures.
Locality comes as a specific feature of regularized
Wasserstein distance. Indeed as made clear by the
solution in Equation \eqref{eq:solregwda} of the entropy-smoothed optimal transport problem, weights $T_{ij}$ tend to be larger for nearby points with
an exponential decrease with respect to distance between $\P\x_i$ and 
$\P\x_{j'}$.

\subsection{Regularized Wasserstein Distance and Fisher criterion.} 
 Fisher criterion for measuring separability
stands on the ratio of inter-class and intra-class variability of samples. 
However, this intra-class variability can be challenging to evaluate
when information regarding probability distributions come only
through empirical examples.  
Indeed, the classical $(\lambda = \infty)$
Wasserstein distance of a discrete
distribution with itself is $0$, as with any other metrics for empirical
distributions.
Recent result by \citet{Mueller15} also suggests that even splitting
examples from one given class and computing Wasserstein distance
between resulting empirical distributions will result
in arbitrary small distance with high probability.
This is why entropy-regularized Wasserstein distance plays a key role
in our algorithm, as to the best
of our knowledge, no other metrics on empirical distributions
would lead to relevant intra-class measures. 
Indeed, $W_\lambda(\P\X,\P\X) = \langle \P^\top \P, \C \rangle
$
with $\C = \sum_{i,j} T_{i,j} (\x_i - \x_j) (\x_i - \x_j)^T$. 
Hence, since $\lambda < \infty$  ensures that  mass of
a given sample is split among its neighbours by the transport map $\T$, $W_\lambda(\P\X,\P\X)$ is thus non-zero and interestingly, it depends on a weighted
covariance matrix $\C$ which, because it depends on $\T$, will  put
more emphasis on couples of neighbour examples. 

More formally, we can show that minimizing  
$W_\lambda (\P\X,\P\X)$ with respect to $\P$
 induces  a neighbourhood preserving map $\P$. 
This means that if an example $i$ is closer to an
example $j$ than an example $k$ in the original
space, this relation should be preserved in the
projected space. \rev{ This implies that 
$\|\P \x_i - \P \x_j\|_2$ should be smaller than
$\|\P \x_i - \P \x_k\|_2$.
Then,
this neighbourhood preservation can be enforced if  
$K_{i,j} > K_{i,k}$, which is equivalent to $M_{i,j} < M_{i,k}$, implies $T_{i,j} > T_{i,k}$. Hence,  since 
$W_\lambda (\P\X,\P\X) = \sum_{i,j} T_{i,j} \|\P \x_i - \P \x_j\|_2^2$,
the inequality  $T_{i,j} > T_{i,k}$  means that examples that
are close in the input space are encouraged to be close
in the projected space. }
We show next that there exists situation in which
this condition is guaranteed. 
\begin{proposition} \label{prop:locality}Given $\T$ the solution of 
the entropy-smoothed optimal transport problem, as
defined in Eq. (\ref{eq:primal_sol}),
between the empirical distribution $\P\X$ on itself, $\forall i,j,k$:
\begin{equation*}
  \exists \alpha\geq 1 , K_{i,j}> \alpha K_{i,k}\quad \Rightarrow
  \quad T_{i,j} > T_{i,k}
\end{equation*}
\end{proposition}
\begin{proof}
Because, we are interested in transporting $\P\X$ on $\P\X$,
the resulting matrix $\K$ is symmetric non-negative. Then,  according to
Lemma 1 (see appendix), the solution matrix $\T$ is also symmetric. Owing to
this symmetry and the properties of the Sinkhorn-Knopp algorithm, it
can be shown \citep{knight2008sinkhorn} that there exists a
non-negative vector $\v$ such that 
$T_{i,j} = K_{i,j} v_i v_j$. In addition, because this vector
is the limit of a convergent sequence obtained by the Sinkhorn-Knopp
algorithm, it is bounded. Let us denoted as $A$, a constant such
that $\forall i, v_i \geq A$. Furthermore, in Lemma 2 (see appendix), we show that
$\forall i, 1 \geq v_i$.   Now, it is easy to prove that
$$
A > \frac{K_{i,k}}{K_{i,j}} \implies v_j \frac{K_{i,k}}{K_{i,j}}
\implies K_{i,j} v_j - K_{i,k} > 0
$$  
and because $\forall j,\, v_j \leq 1$ and by definition 
$T_{i,j} =  v_i K_{i,j} v_j$, we thus have
$
K_{i,j} v_j - K_{i,k}v_k > 0  \implies T_{i,j} > T_{i,k}
$. in Proposition \ref{prop:locality} we take
$\alpha=\frac{1}{A}$ since $A<1$.
\end{proof}

Note that this proposition provides us with a guarantee 
on a  ratio $\frac{K_{i,k}}{K_{i,j}}$ between examples that induces 
preservation of neighbourhood. However, the constant 
$A$ we exhibit here is probably loose and thus,
a larger ratio may still preserve locality.

\subsection{Connection to Fisher Discriminant Analysis.} 
\label{sec:relat-fish-discr}
We show next that WDA encompasses Fisher Discriminant
analysis in the limit case where
$\lambda$ approaches $0$. In this case, we can see from Eq. \eqref{eq:primal_sol}
that the  matrix $\T$ does not depends on the data. The solution $\T$
for each Wasserstein distance is the matrix that
maximizes entropy, namely the 
uniform probability distribution $\T=\frac{1}{nm}\one_{n,m}$.
The cross-covariance matrices become thus
$$
\C^{c,c'}=\frac{1}{n_cn_{c'}}\sum_{i,j}(
  \x^c_i-\x^{c'}_j)( \x^c_i-\x^{c'}_j)^T
$$
and the matrices $\C_w$ and $\C_b$ correspond then to intra- and inter-class
covariances as used in FDA. Since these matrices do not depend
on $\P$, the optimization problem \eqref{eq:wdaprob} boils down to the
usual Rayleigh quotient which can be solved using a generalized
eigendecomposition of $\C_w^{-1}\C_b$ as in FDA. Note that WDA is equivalent to FDA when the classes are balanced (in the unbalanced case one needs to weight the covariances/Wasserstein distances in (\ref{eq:wdaprob}) with the class ratios).
Again, we stress out that beyond this limit case and when $\lambda>0$,
the smoothed optimal transport matrix $\T$ promotes  cross-covariance
matrices that are estimated from local relations as illustrated in Figure
\ref{fig:regwass}.

\rev{Following this connection, we want to stress again the role played
by the within-class Wasserstein distances in WDA. At first, from a theoretical
point of view, optimizing the ratio instead of just maximizing the
between-class distance allows us to encompass well-known method such as FDA.
Secondly, as we have shown in the previous subsection, minimizing
the within-class distance provides interesting features such as 
neighbourhood preservation under mild condition.

Another intuitive benefit of minimizing the within-class distance is the following. Suppose we have several projection maps that lead to the same optimal transport matrix $\T$. Since
$W_\lambda (\P\X,\P\X) = \sum_{i,j} T_{i,j} \|\P \x_i - \P \x_j\|_2^2$ for any
$\P$, minimizing $W_\lambda (\P\X,\P\X)$ with respect to $\P$ means
preferring the projection map that yields to the smaller weighted (according
to $\T$) pairwise distance of samples in the projected space. Since
for an example $i$, $\{T_{i,j}\}_j$ are mainly non-zero among
neighbours of $i$, minimizing  the within-class distance 
favours projection maps that tend to tighly cluster points in
the same class.   
 }

\subsection{Relation to other information-theoretic  discriminant analysis.}
Several information-theoretic criteria have been considered 
for discriminant analysis and dimensionality reduction. Compared
to Fisher's criteria, these ones have the advantage of going
beyond a simple sketching of the data pdf based on second-order
statistics. Two recent approaches are based on the idea
of maximizing distance of probability distributions of data
in the projection subspaces. They just differ in the
choice of the metrics of pdf (one being a $L_2$ distance \citep{emigh2015linear}
and the second one being a Wasserstein distance \citep{Mueller15}).
While our approach also seeks at finding projection that maximizes
pdf distance, it has also the unique feature of finding projections
that preserves neighbourhood.
Other recent approaches have addressed the problem of supervised
dimensionality reduction algorithms still from an information
theoretic learning perspective
but without directly maximizing distance of pdf in the projected
subpaces. 
We discuss two methods to which we have compared
with in the experimental analysis.  The approach of \citet{suzuki2013sufficient}, denoted as LSDR, 
seeks at finding a low-rank subspace of inputs that contains
sufficient information for predicting output values. In their works,
the authors
define the notion of sufficiency through conditional independence
of the outputs and the inputs given the projected inputs and evaluate
this measure through squared-loss mutual information. One major
drawback of their
approach is that they need to estimate a density ratio introducing
thus an extra layer of complexity and an error-prone task. 
Similar idea has been investigated by \citet{tangkaratt2015direct}
as they used quadratic mutual information for evaluating statistical dependence
between projected inputs and outputs (the method has been named
LSQMI). While they avoid the estimation of
density ratio, they still need to estimate derivatives of quadratic
mutual information.
Like our approach, the method of  \citet{giraldo2013information} avoids density estimation
 for performing supervised metric learning. Indeed, the
key aspect of their work is to show that the Gram matrix of some data samples  
can be related to some information theoretic quantities such as
conditional entropy without the need of estimating pdfs. Based
on this finding, they introduced a metric learning approach, coined CEML, by
minimizing conditional entropy between labels and projected
samples. While their approach is appealing, we believe that 
a direct criterion such as Fisher's is more relevant and robust 
for classification purposes, as proved in our experiments.

\subsection{Wasserstein distances and machine learning. }
Wasserstein distances are mainly derived from the theory of optimal transport~\citep{villani08}, and provide a useful way to compare
probability measures. Its practical deployment in machine learning problems has been alleviated thanks to regularized versions
of the original problem
\citep{CuturiSinkhorn,benamou2015iterative}. The geometry of the space
of probability measures endowed with the 
Wasserstein metric allows to consider various objects of interest such
as means or barycenters~\citep{Cuturi14,benamou2015iterative}, and has
led to 
generalization of PCA in the space of probability
measures~\citep{Seguy15}. It has been considered in the problem of
semi-supervised learning~\citep{Solomon14},
domain adaptation~\citep{courty2016optimal}, or definition of loss
functions~\citep{Frogner15}. More recently, it has also been considered
in a subspace identification problem
for analyzing the differences between distributions~\citep{Mueller15},
but contrary to our approach, they only consider projections to
univariate distributions,
and as such do not permit to find subspaces with dimension $>1$.
{More recent works have proposed to use Wasserstein for measuring
  similarity between documents in \citet{huang2016supervised} and
  propose to learn a metric that encodes class information between
  samples. Note that in our work we use Wasserstein between the
  empirical distributions and not the training samples yielding a very
  different approach.

\begin{figure*}[!t]
  \centering
  \includegraphics[width=.7\linewidth]{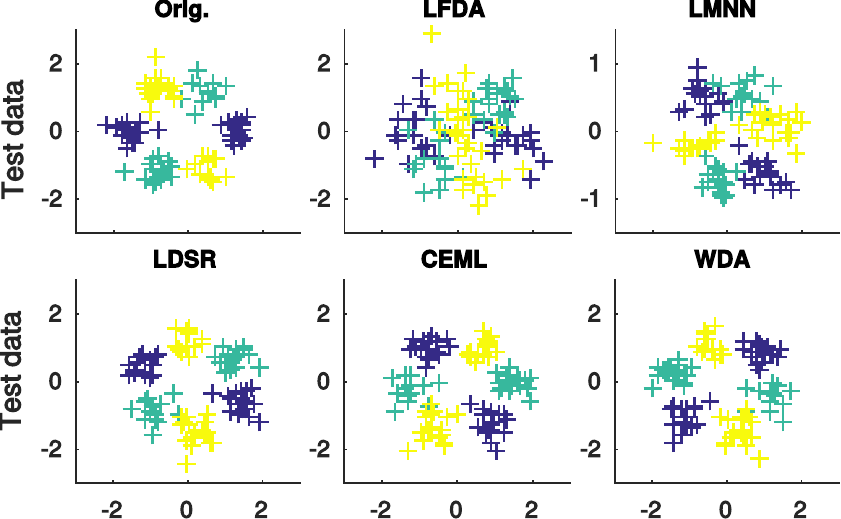}  \caption{Illustration of subspace learning methods on a nonlinearly
    separable 3-class toy example of dimension $d=10$ with 2 discriminant
features (shown on the upper left) and 8 Gaussian noise features. 
Projections onto $p=2$ of the test data  are
reported for several subspace estimation methods.}
  \label{fig:illuswda}
\end{figure*}

\section{Numerical experiments}
\label{sec:numer-exper}

In this section we illustrate how WDA works on several learning problems.
 First, we evaluate our approach on a simple simulated dataset
with a 2-dimensional discriminative
subspace. Then, we benchmark WDA on MNIST
and Caltech datasets {with some pre-defined hyperparameter settings for
methods having some}.
\rev{Unless specified and justified, for LFDA and LMNN, we have set the number of neighours to $5$.
For CEML, Gaussian kernel width $\sigma$ has been fixed to 
$\sqrt{p}$, which is the value used by  \citet{giraldo2013information} 
across all their experiments. For WDA, we have chosen $\lambda = 0.01$
except for the toy problem.}
The final experiment compares performance of WDA and
competitors on some UCI dataset problems, in which relevant
parameters have been validated.

Note that in the spirit of reproducible research  the
 code will be made available to the community and the Python implementation of
 WDA is available as part of the POT for Python Optimal Transport Toolbox \citep{flamary2017pot}  on Github\footnote{Code : \url{https://github.com/rflamary/POT/blob/master/ot/dr.py} }.

\subsection{Practical implementation.}

In order to make the method less sensitive to the
dimension and scaling of the data, we propose to use a
pre-computed adaptive regularization parameter for
each Wasserstein distances in \eqref{eq:wdaprob}. Denote
as $\lambda_{c,c'}$ such parameter yielding thus to a distance $W_{\lambda_{c,c'}}$.
 In practice, we
initialize $\P$ with the PCA projection,  and 
define $\lambda_{c,c'}$ as $\lambda_{c,c'}=\lambda(\frac{1}{n_cn_{c'}}\sum_{i,j}\|\P x_i^c-\P
x_j^{c'}\|^2)^{-1}$ between class $c$ and $c'$. These values are
computed \emph{a priori} and fixed in the remaining iterations.
They have the advantage to promote a similar regularization strength between inter and
intra-class distances.

We have compared our WDA algorithms to some classical dimensionality
reduction algorithms like PCA and FDA, to some locality preserving
methods such as LFDA and LMNN and to some recent mutual information-based
supervised dimensionality and metric learning algorithms such
as LSDR and CEML  mentioned above. For the last
three methods, we have used the author's implementations. We have
also considered LSQMI as a competitor but did not report its performances
as they were always worse than those of LSDR.

\subsection{Simulated dataset.}
\label{sec:simulated-dataset}

\begin{figure*}[t]
  \centering
  \includegraphics[width=.47\linewidth]{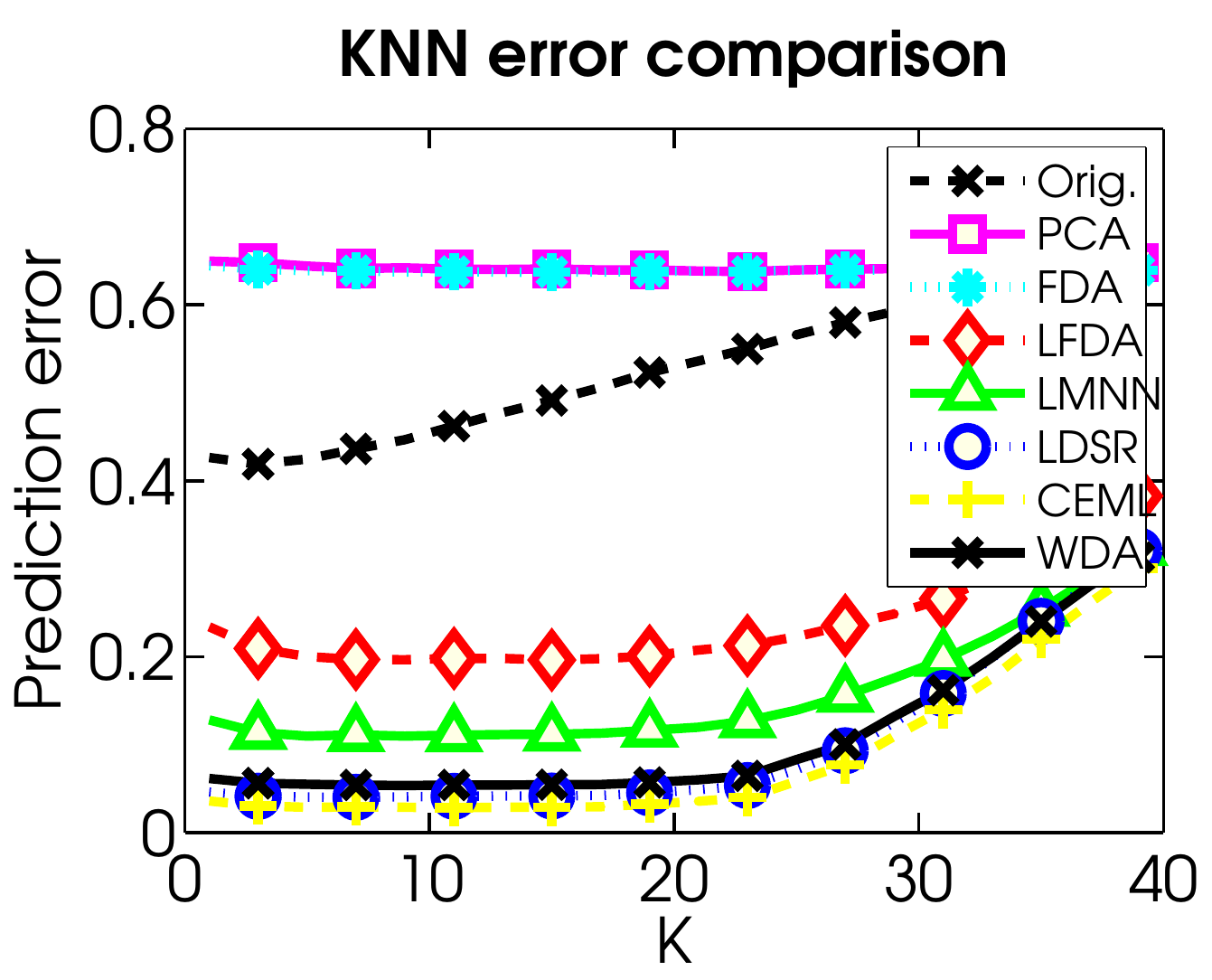}
  \includegraphics[width=.47\linewidth]{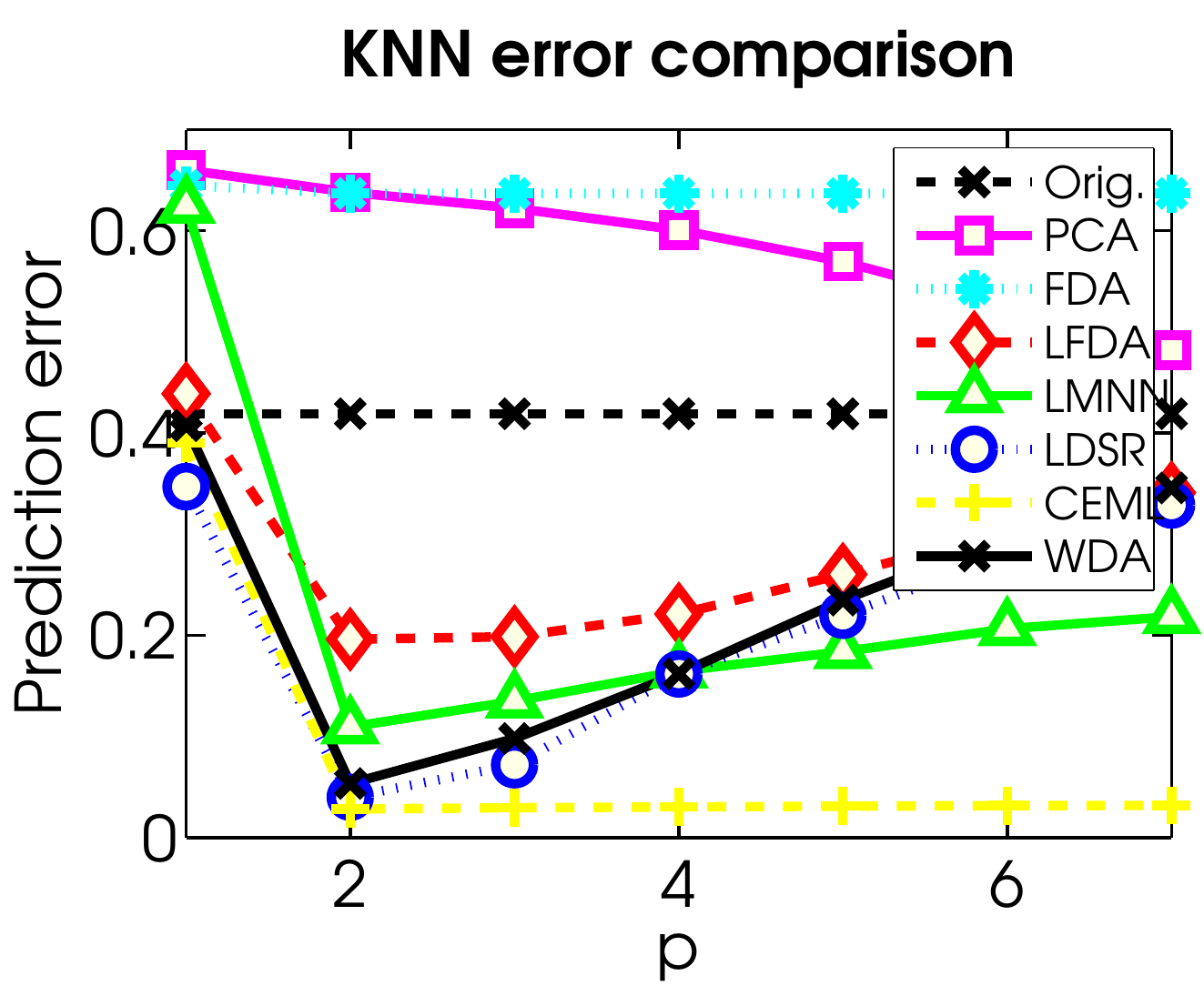}

  \caption{Prediction error on the simulated dataset (left) with projection dimension fixed to  $p=2$ and error for varying $K$ in the KNN classifier.
    (right) evolution of performance with different projection 
dimension $p$ and best $K$ in the KNN classifier.}
  \label{fig:toyerror}
\end{figure*}
\begin{figure}[t]
\centering
\includegraphics[width=.47\linewidth]{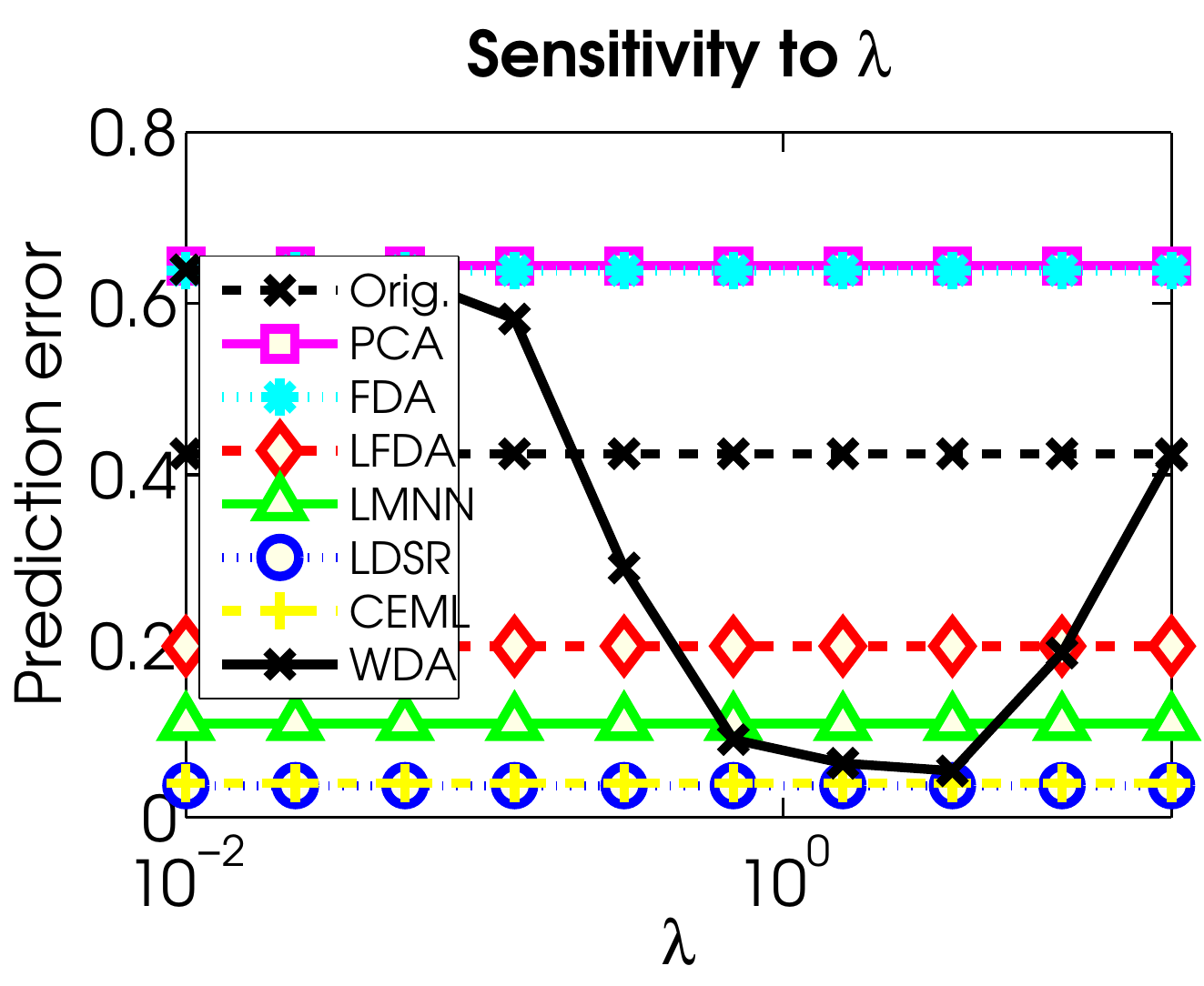}    
\includegraphics[width=.47\linewidth]{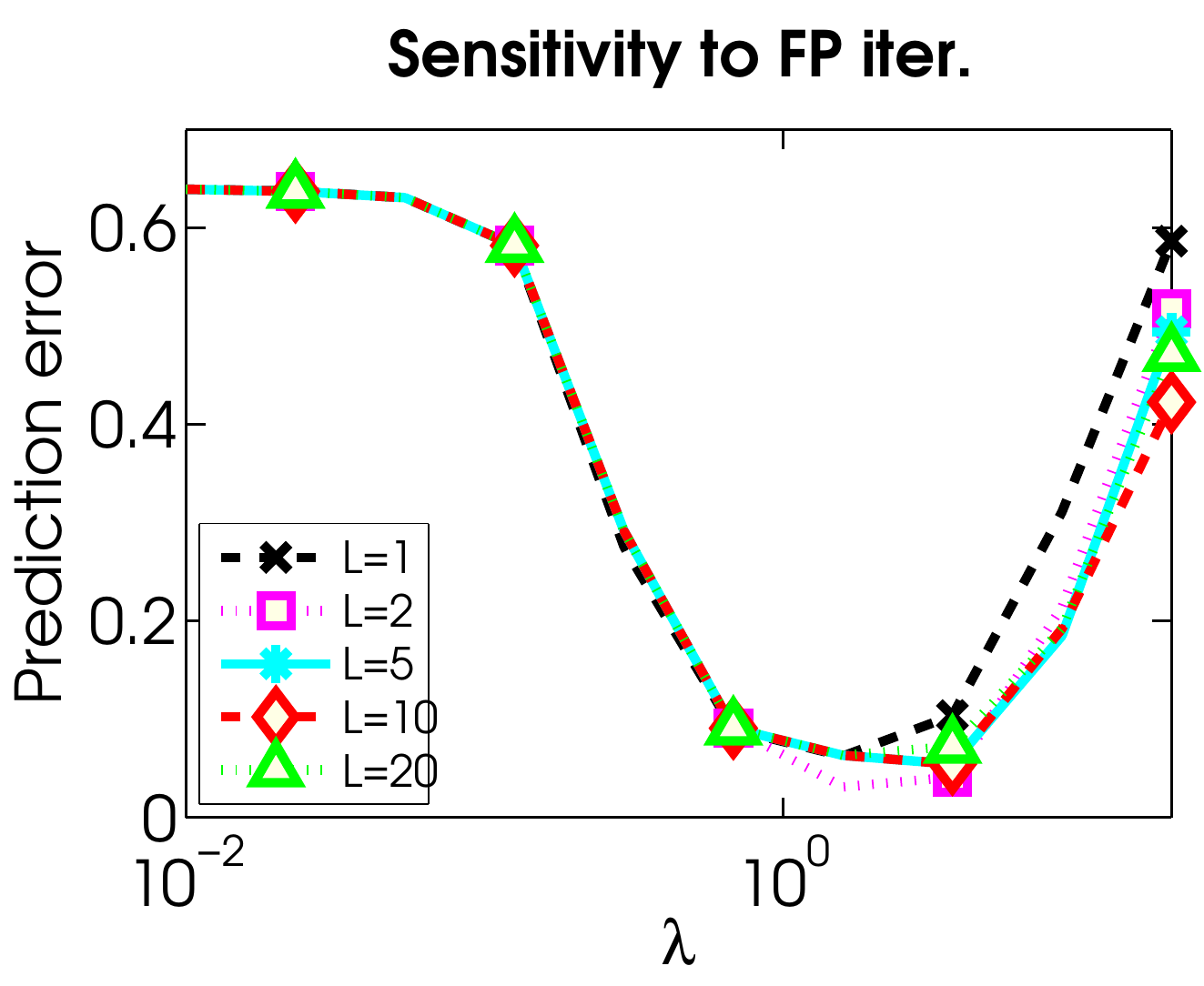}
\caption{ Comparison of WDA performances on the simulated dataset (left) as a function of $\lambda$. (right) as a
  function of $\lambda$ and with different  number of fixed point iterations.}
  \label{fig:toyerror2}
\end{figure}

This dataset has been designed for evaluating the ability of a subspace
method to uncover a discrimative linear subspace when the classes are non-linearly separable. It is a 3-class problem in dimension $d=10$ with
two discriminative  dimensions, the remaining $8$ containing
Gaussian noise. In the $2$ discriminant features, each class is composed
of  two modes as illustrated  in the upper left part of Figure \ref{fig:illuswda}. 

Figure \ref{fig:illuswda} also illustrates the projection 
of test samples in two-dimensional subspaces obtained
from the different approaches.  We can see that for
this dataset WDA, LDSR and CEML lead to a good discriminant
subspace. This illustrate the importance of estimating relations
between samples in the projected space as opposed to the original
space as done in LMNN and LFDA.
Quantitative results are illustrated in Figure \ref{fig:toyerror} (left) where we
reported prediction error for a
K-Nearest-Neighbors classifier (KNN) for $n=100$ training examples and
$n_t=5000$ test examples.  In this
simulation, all prediction errors are averaged over 20 data
generations and the neighbors parameters of LMNN and LFDA have been
selected empirically to maximize performances (respectively $5$ for LMNN and
$1$ for LFDA).  We can see in the left part of the figure that WDA,
LDSR and CEML and to a lesser extent LMNN can estimate the relevant subspace, 
\rev{when the optimal dimension value is given to them},that is robust to the choice of $K$. Note that slightly better performances are achieved by LSDR and CEML. \rev{In the right plot of Figure \ref{fig:toyerror}, we show the performances of all algorithms when varying the dimension of the projected space. 
We note that WDA, LMNN, LSDR and LFDA  achieve their best performances for 
$p=2$ and that prediction errors rapidly increase as $p$ is misspecified. 
Instead, CEML performs very well for $p \geq 2$. Being sensitive to the
correct projected space dimensionality can be considered as an asset, as typically this dimension is to be optimized (\emph{e.g} by cross-validation), making it easier to spot the best dimension reduction. At the contrary, CEML
is robust to projected space dimensionality mis-specification at the expense of 
under-estimating the best reduction of dimension. }

In the left plot for Figure \ref{fig:toyerror2} , we illustrate the
sensitivity of WDA \emph{w.r.t.} the regularization parameter
$\lambda$. WDA returns equivalently good performance on almost a full order of
magnitude of $\lambda$. This suggests that a coarse validation can be performed in practice.  
The right panel of Figure \ref{fig:toyerror2}  shows the performance of the WDA  for different number of inner Sinkhorn iterations $L$. We can see that
even if this parameter leads to different performances for large values
of $\lambda$, it is still possible
find some $\lambda$ that yield near best performance even for small
value of $L$.

\begin{figure*}[t]
  \centering
  \includegraphics[width=.47\linewidth]{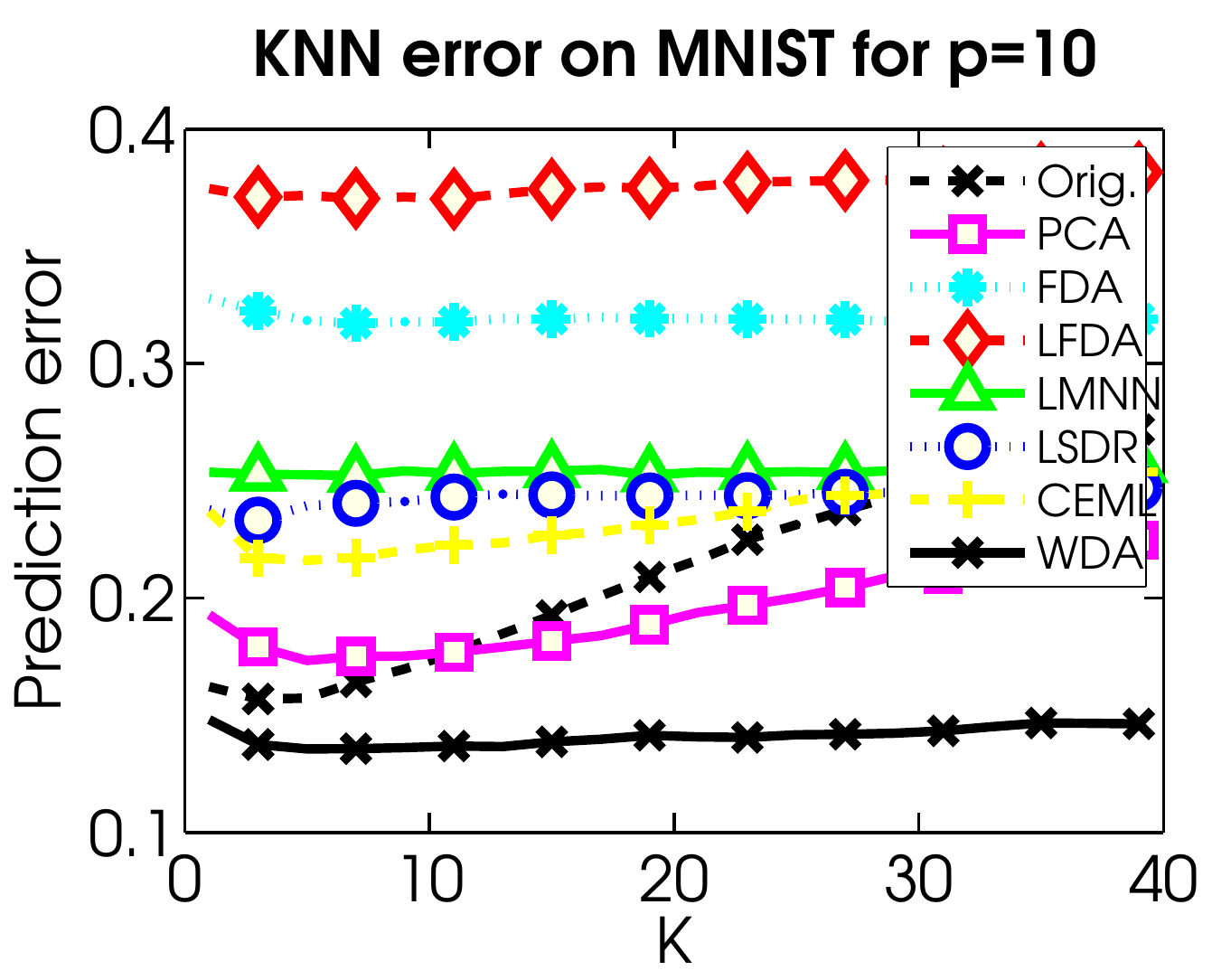}
  \includegraphics[width=.47\linewidth]{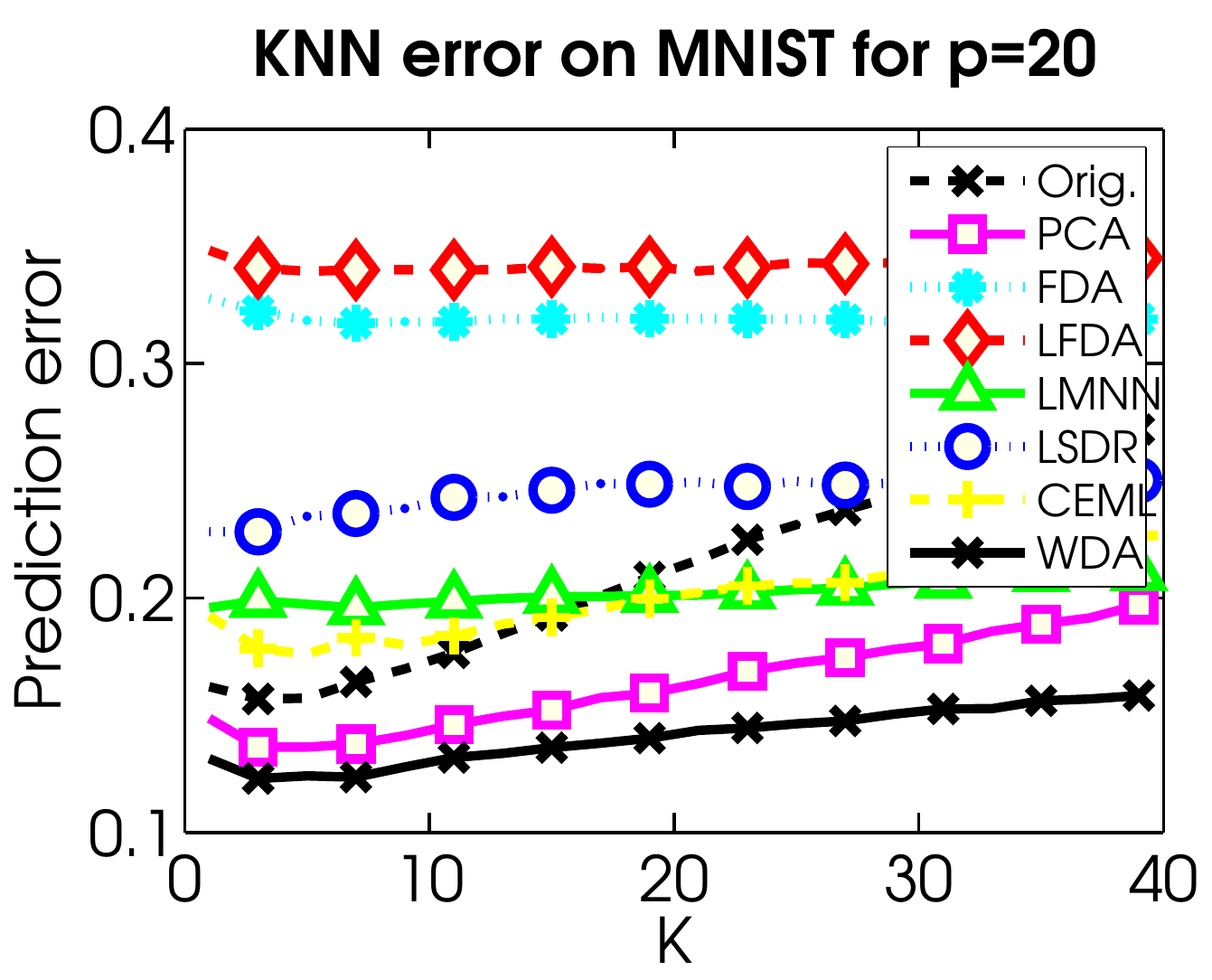}
    \caption{  Averaged prediction error on MNIST with projection dimension 
(left) $p=10$. (right) $p=20$.~ In these plots, LSQMID has
been omitted due to poor performances.}
  \label{fig:vision}
\end{figure*}

\begin{figure*}[t]
  \centering
  \includegraphics[width=.5\linewidth]{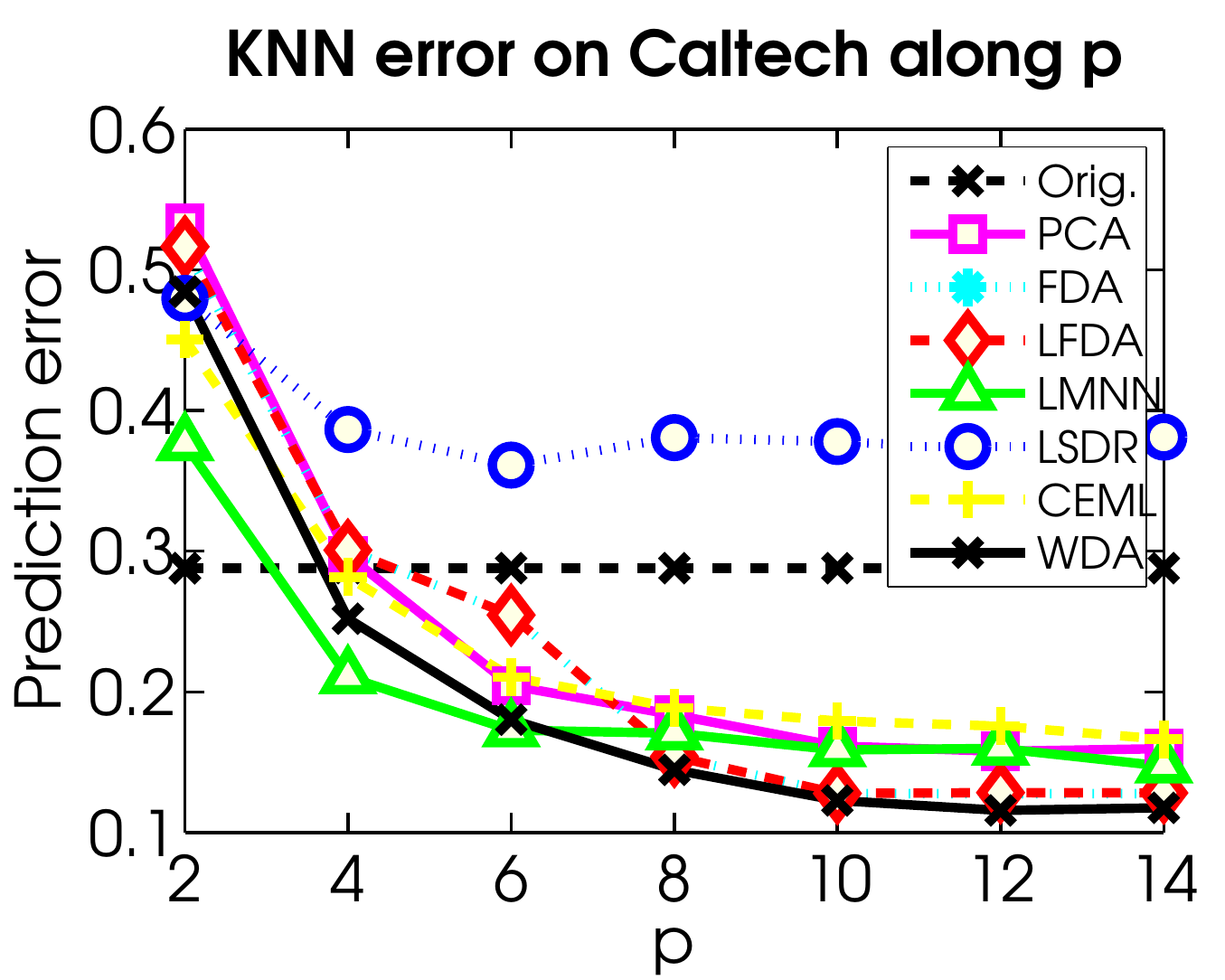}
   \caption{ Averaged prediction error on the Caltech dataset along the
    projection dimension. In these plots, LSQMID has
been omitted due to poor performances.}
  \label{fig:vision2}
\end{figure*}

\begin{figure*}[t]
  \centering
  \includegraphics[width=.9\linewidth]{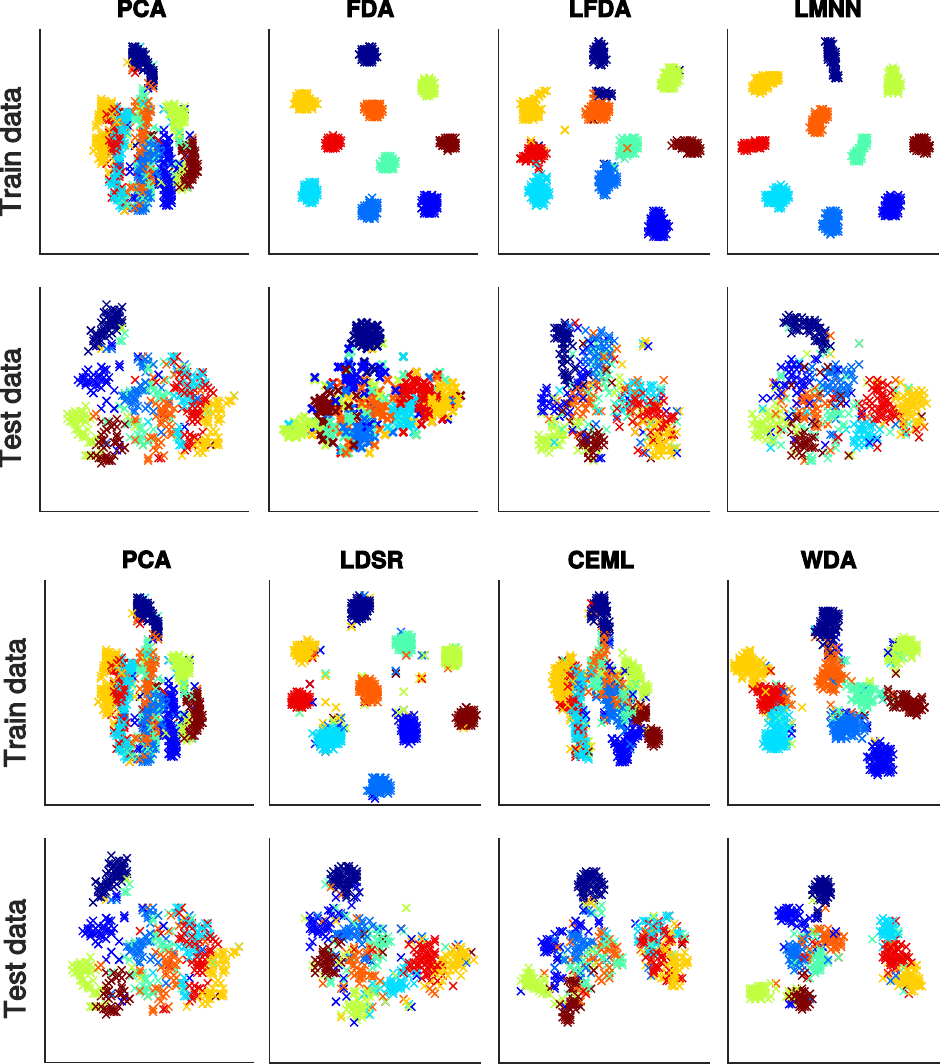}
  \caption{ 2D tSNE of the  MNIST samples projected on $p=10$ for  different
     approaches. (first and third lines) training set (
     second and fourth lines) test set. }
  \label{fig:mnist2}
\end{figure*}

\subsection{MNIST dataset.} 
\label{sec:mnist-dataset}

Our objective with this experiment is to measure how robust our
approach is with only few
training samples despite high-dimensionality of the problem. 
To this end, we draw $n=1000$ samples for training and report the KNN
prediction error as a function of $k$ for the different subspace
methods when projecting onto 
$p=10$ and $p=20$ dimensions (resp. left and right plots of
Figure \ref{fig:vision}). The reported scores are averages of 20
realizations of the same experiment. We also limit the analysis to
$L=10$ as the number of Sinkhorn fixed point iterations and $\lambda=0.01$. 
For both $p$, WDA finds a better 
subspace than the original space which suggests that most of the
discriminant information available in the training dataset has been
correctly extracted. Conversely,  the other approaches struggle to
find a relevant subspace in this
configuration. In addition to better
prediction performance, we want to emphasize that in this configuration,
WDA  leads to a dramatic compression of the data from
784 to 10 or 20 features while preserving  most of the discriminative
information.

To gain a better understanding of the corresponding embedding, we
have further projected the data from the 10-dimensional space to 
a 2-dimensional one using t-SNE~\cite{van2008}. In order
to make  the embeddings comparable, we have used the same
initializations of t-SNE for all methods.
 The resulting  2D projections on the test
samples are
shown in Figure \ref{fig:mnist2}. We can clearly see the
overfitting behaviour of FDA, LFDA, LMNN and LDSR that separate accurately the
training samples but fail to separate the test samples. Instead,
WDA is able to disentangle classes in the training set while
preserving generalization abilities.

\subsection{Caltech dataset.}
 In this experiment, we use a subset described by~\citet{donahue2014}
 of the Caltech-256 image
 collection~\cite{griffin07}. The
dataset uses  
features that are the output of the DeCAF deep learning
architecture~\cite{donahue2014}. More
precisely, they are extracted as
the sparse activation of the neurons from the 6th fully connected
layer of a convolutional network trained on ImageNet and then
fine-tuned for the considered visual recognition task. As such, they
form
vectors of 4096 dimensions and we are looking for subspace as
small as $15$. In this setting, $500$ images are
considered for training, and the remaining portion of the dataset for
testing ($623$ images). There are 9 different classes in this
dataset. We examine in this experiment how the proposed dimensionality
reduction performs when changing the subspace dimensionality. For this problem, the regularization
parameter $\lambda$ of WDA was empirically set to $10^{-2}$.
The K in KNN was  set to $3$ which is a common standard setting for this
classifier. The reported results reported in Figure~\ref{fig:vision2} are averaged over 10 realizations of
the same experiment. When $p\geq 5$, WDA already finds a subspace which
gathers relevant discriminative information from the original space. In
this experiment, LMNN yields to a better subspace for small $p$ values
while WDA is the best performing method for $p\geq 6$.  Those results
highlight the potential
interest for using WDA as linear dimensionality reduction layers in
neural-nets architecture. 

\subsection{Running-time}

For the above experiments on MNIST and Caltech, we have also evaluated the running
times of the compared algorithms. The LFDA, LMNN, LDSR and CEML codes are the Matlab code
that have been released by the authors. Our WDA code is Python-based and relies on 
the  POT toolbox \cite{}. All these codes have been runned on a $16$-core Intel Xeon E5-2630
CPU, operating at 2.4 GHz with GNU/Linux and 144 Gb of RAM. 

Running times needed for computing learned subspaces are reported in Table \ref{tab:timing}. We first remark that LSDR
is not scalable. For instance, ot needs  several tenths of hour for computing the projection from  $4096$ to $14$
dimensions on Caltech. More generally, we can note that our WDA algorithm
scales well and is cheaper to compute than LMNN and is far less expensive than
CEML on our machine. We believe our WDA algorithm  better leverages multi-core machines
owing the large amount of matrix-vector multiplications needed for
computing Sinkorhn iterations.
   
\begin{table}[t]
  \centering\footnotesize\setlength{\tabcolsep}{4pt}
\resizebox{.99\columnwidth}{!}{
  \begin{tabular}[h]{lccccccc}
\hline
 Datasets  &   PCA& FDA& LFDA& LMNN& LSDR& CEML& WDA\\\hline \hline
Mnist (10) &  0.39(0.1)& 0.69(0.2)& 0.55(0.4)& 20.55(14.2)& 29813(5048)& 87.02(8.7)& 6.28(0.3)\\\hline
Mnist (20) &  0.38(0.0)& 0.58(0.0)& 0.54(0.2)& 18.27(17.0)& 60147(11176)& 90.22(8.8)& 6.15(0.1)\\\hline
Caltech (14)&  0.53(0.3)& 21.38(6.1)& 11.43(2.0)& 39.56(6.3)& 140776(53036)& 14.59(7.6)& 5.29(0.1)\\\hline
  \end{tabular}}
  \caption{Averaged running time in seconds of the different
    algorithms for computing 
the learned subspaces.}
  \label{tab:timing}
\end{table}

\subsection{UCI datasets.}

 \begin{table}[t]
\setlength{\tabcolsep}{4pt}
   \centering
 
\footnotesize 
  \begin{tabular}[h]{lccccccccc}
\hline
Datasets & Orig.& PCA& FDA& LFDA& LMNN& LSDR& LSQMI& CEML& WDA\\\hline
wines & 24.33& 26.57& 37.87& 29.21& 32.81& 32.81& 46.29& \textbf{15.34}& \underline{16.91}\\\hline
iris & 42.07& 40.60& \textbf{19.27}& 25.13& \underline{21.67}& 37.93& 56.27& \underline{20.87}& \underline{20.87}\\\hline
glass & 54.01& 58.16& 57.45& 59.53& 54.25& 50.85& 65.42& \textbf{34.86}& 45.99\\\hline
vehicles & 58.68& 57.26& \underline{48.57}& \underline{48.25}& \textbf{40.84}& 51.86& 65.09& \underline{48.46}& 51.13\\\hline
credit & 28.90& 25.57& 18.67& \underline{17.69}& 23.73& 24.71& 39.01& \underline{17.65}& \textbf{17.39}\\\hline
ionosphere & 26.14& 26.90& 29.63& 27.64& 30.80& 31.08& 36.42& \underline{22.87}& \textbf{20.40}\\\hline
isolet & 17.50& 17.60& 15.12& 13.96& \textbf{11.13}& 13.33& 21.76& 30.19& 14.41\\\hline
usps & 7.59& 7.66& 11.63& 12.76& \textbf{6.05}& 8.77& 14.83& 10.15& 6.50\\\hline
mnist & 17.26& 14.16& 33.85& 29.92& \underline{13.95}& 26.53& 60.05& 24.68& \textbf{13.07}\\\hline
caltechpca & 23.39& 13.93& 12.03& 18.19& \underline{11.55}& 36.08& 100.00& 13.65& \textbf{11.45}\\\hline
Aver. Rank& 5.4 & 5.5 & 5.2 & 5.2 & 3.4 & 5.7 & 8.9 & 3.5 & 2.2

  \end{tabular}
  \caption{Average test errors over $20$ trials on UCI datasets. 
In bold, the lower test error accross algorithms. Underlined averaged test
errors that are statistically non-significantly different according to
a signrank test with p-val = 0.05. Result of LSQMI on caltech has not been reported due to lack of convergence after few days of computation. \label{ucitable}}

\end{table}

We have also compared the performances of the dimensionality reduction
algorithms on some UCI benchmark datasets \cite{Lichman:2013}. The experimental
setting is similar to the one proposed by the authors of LSQMI
\cite{tangkaratt2015direct}. For these UCI datasets, we have appended the
original input features
with some noise features of dimensionality $100$. 
We have split the examples $50\%-50\%$ in
a training and test set. Hyper-parameters such as the number
of neighbours for the KNN and and the dimensionality of the 
projection has been cross-validated on the training set
and choosed respectively among the values $[1:2:19]$ (in Matlab notation)
and $[5, 10, 15, 20, 25]$. Splits have been performed $20$ times.
Note that we have also added experiments with Isolet, USPS, MNIST and
Caltech datasets under
this validation setting but without the additional noisy features.  
Table \ref{ucitable} presents the performance of competing methods.
We note that  our WDA is more robust than all other methods and
is able to capture relevant information in the learned subspaces. 
Its average ranking on all datasets is $2.2$ while the second best,
LMNN is $3.4$. 
There is only one dataset (\emph{vehicles}) for which WDA performs
significantly worse than top methods. Interestingly LSDR and LSQMI
seem to  be less robust than
LMNN and FDA, against which they have not been compared in the
original paper \citep{tangkaratt2015direct}.

\section{Conclusion}
\label{sec:conclusion}
This work presents the Wasserstein Discriminant Analysis, a new and
original linear discriminant subspace estimation method. Based on the
framework of regularized Wasserstein distances, which measure a global
similarity
between empirical distributions, WDA operates by separating
distributions of different classes in the subspace, while maintaining
a coherent structure at a class level. To this extent, the use of
regularization in the Wasserstein formulation allows to effectively
bridge a gap between a global coherency and the local structure of the
class manifold.  This comes at a cost of a difficult optimization of a
bi-level program, for which we  proposed  an efficient method based on
automatic differentiation of the Sinkhorn algorithm.
Numerical experiments show that the method
performs well on a variety of features, including those obtained with
a deep neural architecture. Future work will consider stochastic
versions of the same approach in order to enhance further the ability
of the method to handle large volume of high-dimensional data.

\bibliographystyle{spbasic}      
   
\appendix

\section{Illustration of the transport $\T^{c,c'}$}
\label{sec:illus}
In this Section, we provide intuition on how the transport $\T^{c,c'}$ between class $c$ and $c'$ behaves in 2D toy problem.
Remind that this matrix plays an essential role on how the covariance matrix
$\C$ is estimated in Equation (6).  

In this example, illustrated in Figure~\ref{fig:visu1}, two bi-modal
Gaussian distributions are sampled to produce two distributions
representing two classes. We illustrate in Figure~\ref{fig:tvisu} the
transport $\T^{1,2}$ (inter-class) and $\{\T^{1,1}\T^{2,2}\}$
(intra-class). The corresponding transport matrices are either
displayed in matrix form as inserts, or as connections between the
samples. Those connections have a width parametrized by the magnitude
of the connection (i.e. a small $t_{i,j}$ value will be displayed as a
very thin connection). We note that for visualization purpose, the
magnitude of the $\T$ elements displayed in matrix form are normalized
by the the largest magnitude in the matrix. The transport maps can be
observed in Figure~\ref{fig:tvisu} for three different values of the
$\lambda$ parameter ($\lambda=1,0.5,0.1$). One can notice the locality
induced by large values of $\lambda$, which allows to concentrate the
connections on specific modes of the distributions. When $\lambda$ is
smaller, inter-modes connections start to appear, which allows to
consider the data distributions at a larger scale when computing $\C$. Regarding the
inter-class transport $\T^{1,2}$ , one can also observe the specific
relations induced by the optimal transport maps, that do not
associate modes together, but rather dispatch one mode of each class
onto the two modes of the other.

\begin{figure}[pt]
  \centering
  \includegraphics[width=0.5\columnwidth]{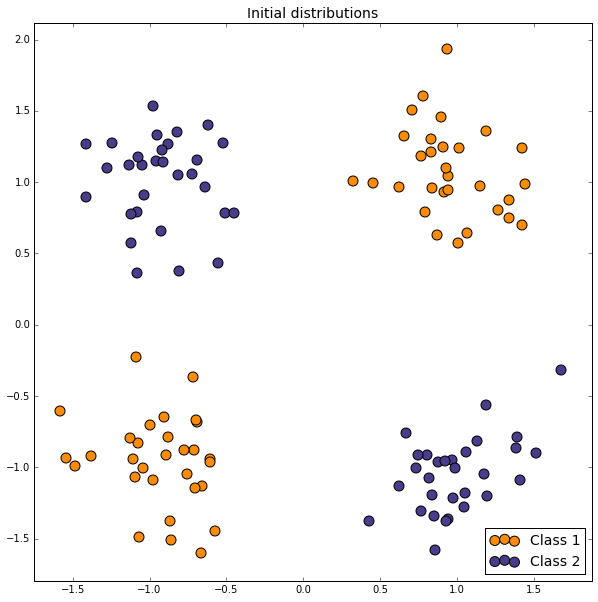}
 \caption{Illustration of the evolution of the transport for two classes $c=1$ and $c'=2$ }
  \label{fig:visu1}
\end{figure}

\begin{figure*}[htbp]
  \centering
  \includegraphics[width=.45\linewidth]{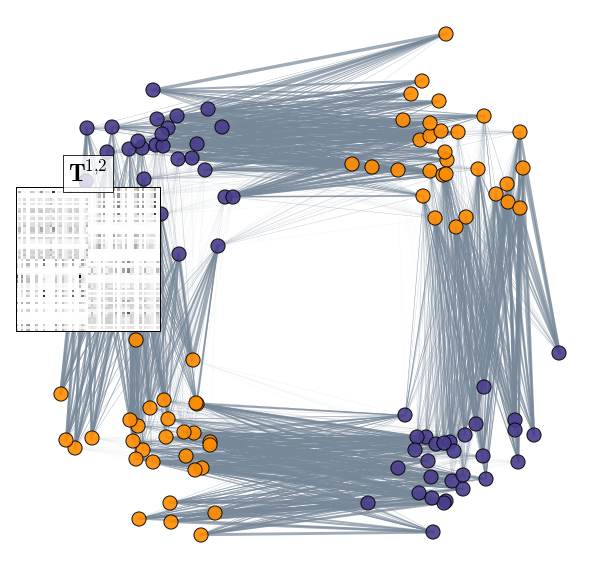}
  \includegraphics[width=.45\linewidth]{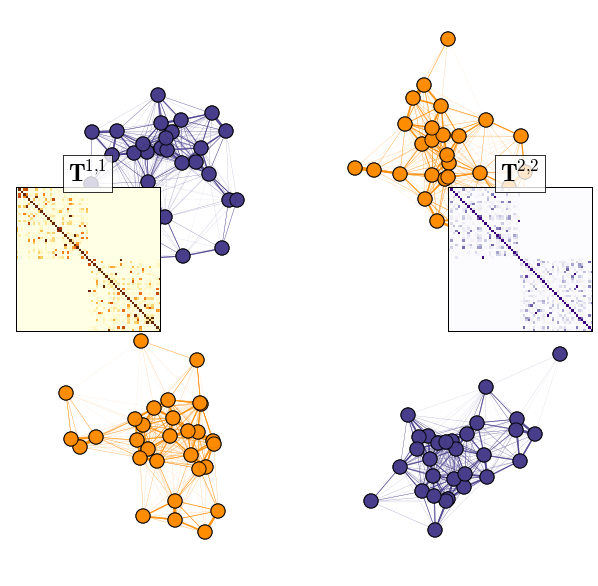}
    \centering
  \includegraphics[width=.45\linewidth]{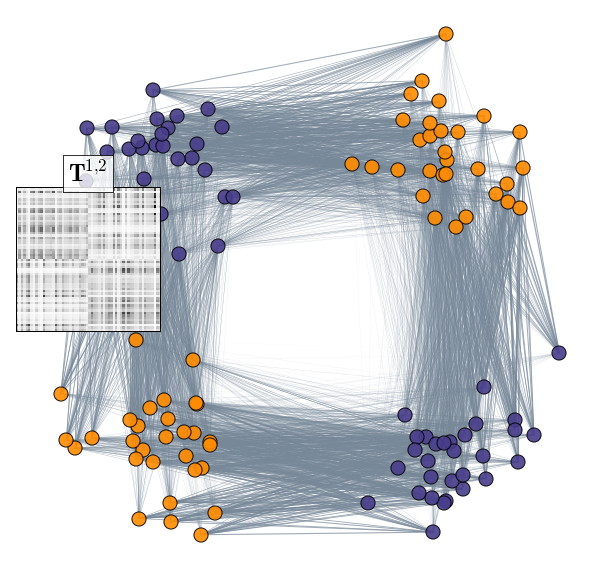}
  \includegraphics[width=.45\linewidth]{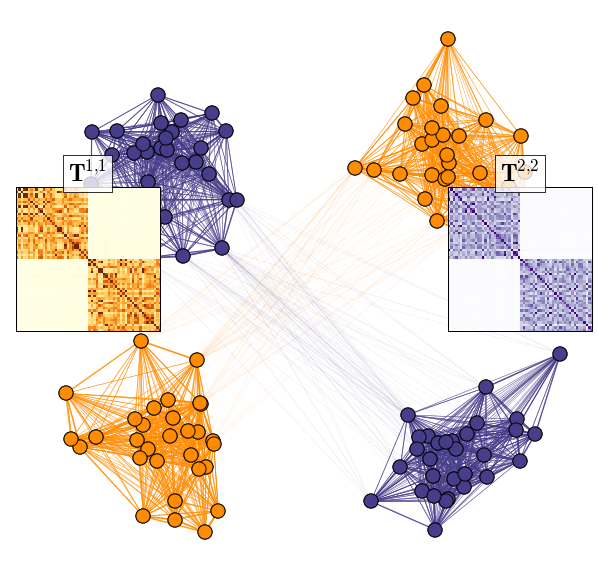}
  \centering
  \includegraphics[width=.45\linewidth]{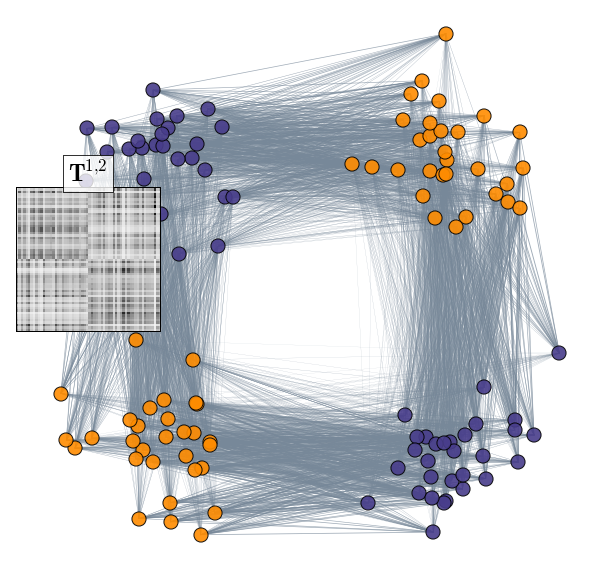}
  \includegraphics[width=.45\linewidth]{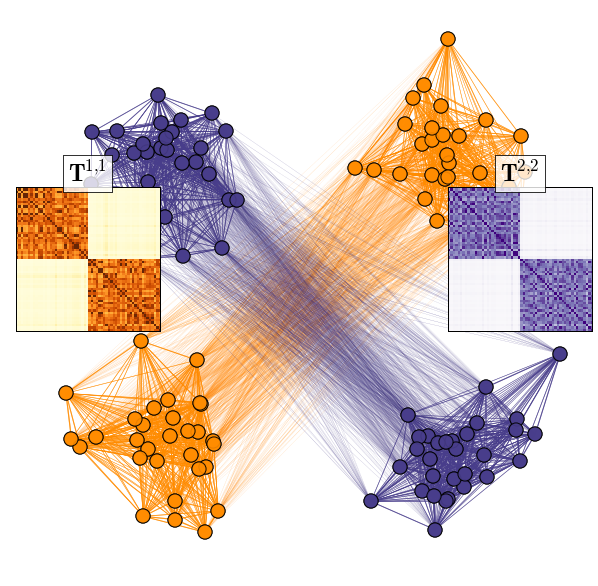}

 \caption{Illustration of the evolution of the transport for two classes for three values of the $\lambda$ parameter (first row) $\lambda=1$ (second row) $\lambda=0.5$ (last row) $\lambda=0.1$. The left column illustrates inter-class relations, while the right column illustrates intra-class relations.   }
  \label{fig:tvisu}
\end{figure*}

\section{Implicit function gradient computation}
\label{sec:deta-impl-funct}

 In this
section, we propose to compute this Jacobian based on the implicit
function theorem.

{For clarity's sake, in this subsection we will not use
the $c,c'$ indices
and $\T$ represents an optimal transport matrix between $n$ and $m$
samples projected with $\P$. First, we express the function $\T(\P)$ as
an implicit function using the optimality conditions of
the equation defining the optimal $\T$ in Equation 6. The Lagrangian of this problem can be expressed as
\cite{CuturiSinkhorn}:
\begin{align*}\small
\mathcal{L}=&\sum_{i,j}\left(
  t_{i,j}m_{i,j}(\P)+t_{i,j}\log(t_{i,j})\right)\\&+\sum_i\alpha_i\left(\sum_jt_{i,j}-r_i\right)+\sum_j\beta_j\left(\sum_it_{i,j}-c_j\right)
\end{align*}
where $\balpha$ and $\bbeta$ are the dual variables associated to the
sum constraints, $m_{i,j}=\|\P\x_i-\P\z_j\|^2$ and in our particular
case $r_i=\frac{1}{n}$ and
$c_j=\frac{1}{m},\forall i,j$. One can define an implicit fonction $g(\P,\T,\balpha,\bbeta):
\dbR^{p\times d + n\times m+ n+
   m}\rightarrow \dbR^{n\times m+ n+
   m} $ from the above lagrangian by
computing its gradient \emph{w.r.t.} ($\T,\balpha,\bbeta$) 
and setting
it to zero owing to optimality. 
The implicit function theorem gives us the following relation:
$$\nabla_\P g=\frac{\partial g}{\partial \P}+\frac{\partial
  g}{\partial \T}\frac{\partial \T}{\partial \P}+\frac{\partial
  g}{\partial \balpha}\frac{\partial \balpha}{\partial
  \P}+\frac{\partial g}{\partial \bbeta}\frac{\partial
  \bbeta}{\partial \P}=\mathbf{0}$$
which can be reformulated as 
\begin{equation}
\begin{bmatrix}\displaystyle
 \frac{\partial \T}{\partial \P}\\\displaystyle\frac{\partial \balpha}{\partial
  \P}\\\displaystyle\frac{\partial
  \bbeta}{\partial \P}
\end{bmatrix}=-\E^{-1}\frac{\partial g}{\partial \P},\quad\text{ with }\quad\E=\begin{bmatrix}\displaystyle
  \frac{\partial
  g}{\partial \T}&\displaystyle\frac{\partial
  g}{\partial \balpha}&\displaystyle\frac{\partial g}{\partial \bbeta}
\end{bmatrix}\label{eq:implicitgrad}
\end{equation}
when the function is well defined and $\E$ is invertible. The derivative
$\frac{\partial \T}{\partial \P}$ can be deduced from the upper part of
the term on the left. Note that all the partial derivatives in
Equ.~\eqref{eq:implicitgrad} are easy to compute. Additionally, $\E$
is a $(pd+nm+n+m)\times (pd+nm+n+m)$ matrix
which is very sparse, as shown in the sequel. However,
assuming for instance 
that the number of points in each class $m=n$ is the same, using this
technique would amount to solve a large $n^2\times n^2$ linear
system with a worst case complexity of $O(n^6)$.}

We now detail the computation of the gradient using the
implicit function theorem. Note that we use the notation of the paper
and that we want to compute the Jacobian  $\frac{\partial \T}{\partial \P}$. First we compute
the implicit function  $g(\P,\T,\balpha,\bbeta): \dbR^{p\times d + n\times m+ n+
  m}\rightarrow \dbR^{n\times m+ n+
  m} $ from the Lagrangian function given in the paper by computing the OT problem optimality conditions:
\begin{align*}
  \frac{\partial\mathcal{L}}{\partial
  t_{k,l}}&=\lambda (\x_k-\z_l)^\top\P^\top\P(\x_k-\z_l)+\log(t_{k,l})\\
           &\quad+1+\alpha_k+\beta_l=0
                                                    \\
 \frac{\partial\mathcal{L}}{\partial
 \alpha_i}&=\sum_jt_{i,j}-r_i=0\\
 \frac{\partial\mathcal{L}}{\partial
 \beta_j}&=\sum_it_{i,j}-c_j=0 
\end{align*}
$ \forall k,l,i,j$. The Jacobian $\frac{\partial \T}{\partial \P}$ can
be computed using the implicit function by solving the following
linear problem:
\begin{equation}
\begin{bmatrix}\displaystyle
 \frac{\partial \T}{\partial \P}\\\displaystyle\frac{\partial \balpha}{\partial
  \P}\\\displaystyle\frac{\partial
  \bbeta}{\partial \P}
\end{bmatrix}=-\E^{-1}\frac{\partial g}{\partial \P},\quad\text{ with }\quad\E=\begin{bmatrix}\displaystyle
  \frac{\partial
  g}{\partial \T}&\displaystyle\frac{\partial
  g}{\partial \balpha}&\displaystyle\frac{\partial g}{\partial \bbeta}
\end{bmatrix}\label{eq:linearsys}
\end{equation}

First $\t=\text{vec}(\T)$ is vectorized as in Matlab with column major format.
\begin{align}
   \frac{\partial
  g}{\partial \T}=
  \begin{bmatrix}
   \text{diag}(\frac{1}{\t})\\\I_n \I_n,\dots,\I_n\\\L_{m,n}^1 \L_{m,n}^2,\dots,\L_{m,n}^m
  \end{bmatrix}
\end{align}
where $\L_{m,n}^k$ is  $\dbR^{m\times n}$ matrix of $0$ with all coefficients on line $k$
equal to 1. 
\begin{align}
   \frac{\partial
  g}{\partial \balpha}=
  \begin{bmatrix}
    {\I_n} \\  {\I_n}\\\hdots\\ {\I_n}
\\\zero_{n,n}\\\zero_{m,n}
  \end{bmatrix}\quad\quad\quad   \frac{\partial
  g}{\partial \bbeta}=
  \begin{bmatrix}
    {\L_{m,n}^1}^\top \\  {\L_{m,n}^2}^\top\\\hdots\\{\L_{m,n}^m}^\top
\\\zero_{n,n}\\\zero_{m,n}
  \end{bmatrix}
\end{align}
Now we compute the last element $\frac{\partial g}{\partial \P}$ using
a vectorization $\p=\text{vec}(\P)$ and $\Delta_{i,j}=\x_i-\z_j$
First note that 
$$\frac{\partial \x^\top\P^\top\P\x}{\partial
  p_{m,l}}=2x_l\sum_ix_ip_{m,i}=2x_l(\P(m,:)\x)$$
which leads to the following Jacobian
\begin{align}
  \frac{\partial g}{\partial \P}=2\lambda
  \begin{bmatrix}
\Delta_{1,1}^\top   \otimes \Delta_{1,1}^\top\P^\top  \\
 \Delta_{2,1}^\top   \otimes  \Delta_{2,1}^\top\P^\top \\
\hdots\\
 \Delta_{n,m}^\top   \otimes  \Delta_{n,m}^\top\P^\top \\
\zero_{n+m,dp}
  \end{bmatrix}
\end{align}
where the upper part of the matrix can be seen as a column-only
Kroenecker product between $\Delta$ and $\P\Delta$.

All the elements are now in place for the linear system \eqref{eq:linearsys},
which can be solved using any efficient method for sparse linear
system. 

\section{Lemmas}

\begin{lemma}
If the matrix $\M \in \mathbb{R}^{n \times n }$ is non-negative symmetric then
the matrix $\T$ defined as in
$$
\argmin_{\T \in U_{n,n}} \lambda \langle \T, \M \rangle - \Omega(\T)
$$
is also symmetric non-negative. Here, $\Omega$ is the entropy of the
matrix $\T$ 
\end{lemma}
\begin{proof} As this optimization problem is strictly convex for
$\lambda < \infty$, and thus admits an unique solution. We show in the sequel that $\T^\top$ achieves the same objective value than $\T$ and thus $\T^\top$ is 
also a minimizer, which naturally leads to $\T^\top = \T$.

First note that the constraints are symmetric thus, $\T^\top$ is feasible.
In addition because the entropy only depends on single entries of the
matrix hence $\Omega(\T) = \Omega(\T^\top)$. Finally, 
\begin{align}\nonumber
\langle \T^\top, \M \rangle& = \sum_{i,j} M_{i,j} T^\top_{i,j}
= \sum_{i,j} M_{i,j} T_{j,i} = \sum_{i,j} M_{j,i} T_{j,i} \\\nonumber
&=  \langle \T, \M \rangle \nonumber
\end{align}
which proves that both matrices lead to the same objective values.
\end{proof}

\begin{lemma}
Suppose that $\T$ is the solution of an entropy-smoothed optimal
transport problem, with matrix $\K$ being symmetric and 
such that $\forall i, K_{i,i}= 1$. There exists
a vector $\v$ such that $\forall i,j, \T_{i,j} = \K_{i,j} \v_i \v_j$
and $\forall i, \v_i \leq 1$.  
\end{lemma}

\begin{proof} Existence of the $\v$ such that $\T_{i,j} = \K_{i,j} \v_i \v_j$
 comes from the fact that the optimization problem can be solved
using the Sinkhorn-Knopp algorithm.  Through the constraints of the
optimal transport problem, we have
$$ \forall i,j\,\,\T_{i,j} =  \K_{i,j} \v_i \v_j \leq \frac{1}{n}.$$
When, $i=j$, as $\K_{i,i} = 1$, we have $\v_i^2 \leq \frac{1}{n}$
and thus $\v_i \leq 1$. 
\end{proof}

\end{document}